\documentclass[10pt]{article}
\usepackage{floatrow}

\usepackage{amsfonts,amsmath,amsthm}
\usepackage{algorithm,algorithmic}
\usepackage{bm}
\usepackage[dvipdfmx]{color}
\usepackage[dvipdfmx]{hyperref}
\usepackage[dvipdfmx]{graphicx}
\usepackage[margin=1in]{geometry}
\usepackage{setspace}

\usepackage{url}
\usepackage{subfigure}

\usepackage{mcr}
\newtheorem{theorem}{Theorem}
\newtheorem{lemma}{Lemma}

\usepackage{cite}
\usepackage{wrapfig}
\usepackage{multirow}
\usepackage{multicol}
\usepackage{ascmac}
\usepackage{comment}

\newcommand{\argmax}{\mathop{\rm argmax}\limits}
\newcommand{\argmin}{\mathop{\rm argmin}\limits}

\title{Efficiently Bounding Optimal Solutions \\ after Small Data Modification \\ in Large-Scale Empirical Risk Minimization}
\author{
Hiroyuki Hanada \\
Nagoya Institute of Technology \\
Nagoya, Aichi, Japan \\
{\tt hanada.hiroyuki@nitech.ac.jp} \\
\and
Atsushi Shibagaki \\
Nagoya Institute of Technology \\
Nagoya, Aichi, Japan \\
{\tt shibagaki.a.mllab.nit@gmail.com} \\
\and
Jun Sakuma \\
University of Tsukuba \\
Tsukuba, Ibaraki, Japan \\
{\tt jun@cs.tsukuba.ac.jp} \\
\and
Ichiro Takeuchi \\
Nagoya Institute of Technology \\
Nagoya, Aichi, Japan \\
{\tt takeuchi.ichiro@nitech.ac.jp} \\
}
\date{June 1, 2016}

\begin{document}
\maketitle

\begin{abstract} 
We study large-scale classification problems in changing environments where a small part of the dataset is modified, and the effect of the data modification must be quickly incorporated into the classifier. 
When the entire dataset is large, even if the amount of the data modification is fairly small, the computational cost of re-training the classifier would be prohibitively large.  
In this paper, we propose a novel method for efficiently incorporating such a data modification effect into the classifier without actually re-training it. 
The proposed method provides bounds on the unknown optimal classifier with the cost only proportional to the size of the data modification.
We demonstrate through numerical experiments that the proposed method provides sufficiently tight bounds with negligible computational costs, especially when a small part of the dataset is modified in a large-scale classification problem.

\end{abstract}

%---------------------------------------
%%%%%%%%%%%%%%%%%%%%%%%%%
%%%%%%%%%%%%%%%%%%%%%%%%%
\section{Introduction} \label{sec:intro}
%%%%%%%%%%%%%%%%%%%%%%%%%
%%%%%%%%%%%%%%%%%%%%%%%%%

In this paper we study a problem of training a classifier such as the support vector machine (SVM) with a large-scale dataset.
When a trained classifier is used in a changing environment where a part of the dataset is modified, we need to accordingly update the classifier for incorporating the effect of the data modification.
However, when the entire dataset is large, even if the amount of the data modification is fairly small, the computational cost of re-training the classifier would be prohibitively large. 
In such a situation of modifying only a tiny part of the dataset, we would expect that the optimal classifier trained with the modified dataset (called \emph{new} classifier) would not be so different from the original classifier which was trained with the original dataset (called \emph{old} classifier).
Thus we might feel that spending a great amount of computational cost for re-training the classifier is not well-worthy effort, and feel like to use the old classifier without caring the effect of the data modification.

In this paper, we propose a novel efficient method for properly incorporating the effect of the data modification without actually re-training the classifier. 
The proposed method
%enables us to obtain lower and upper
provides bounds on the new classifier with the cost only proportional to the size of the data modification. 
This computational advantage is beneficial especially when the size of the data modification is much smaller than the size of the entire dataset. 
The bounds obtained by the proposed method can be used for various tasks in classification problems. 
The bounds can be also used for computing the upper bound of the difference of the old and the new classifiers, which would be helpful for deciding when we should actually re-train the classifier. 

Concretely, consider a linear binary classification problem with $n$ instances and $d$ features, and denote the data matrix as $X \in \RR^{n \times d}$. 
\figurename~\ref{fig:illustration} illustrates the three scenarios considered in this paper.
Denoting the set of modified elements in $X$ as
$\cM$, we are particularly interested in the situation where only a small part of $X$ is modified, i.e., $|\cM|$ is much smaller than $nd$.
Let $\hat{\bm {w}}^* \in \RR^d$ and $\tilde{\bm {w}}^* \in \RR^d$ be the linear model parameters of the old and new classifiers, respectively. 
The proposed method provides a lower bound $L(\tilde{{w}}_j)$ and an upper bound $U(\tilde{{w}}_j)$ such that
\begin{align}
 \label{eq:beta-bounds}
 L(\tilde{{w}}^*_j) \le \tilde{{w}}^*_j \le U(\tilde{{w}}^*_j) \text{ for each } j = 1, \ldots, d 
\end{align}
with the cost only proportional to the number of the modified elements $|\cM|$. 

\begin{figure}[t]
 \begin{center}
  \includegraphics[width=1.0\linewidth]{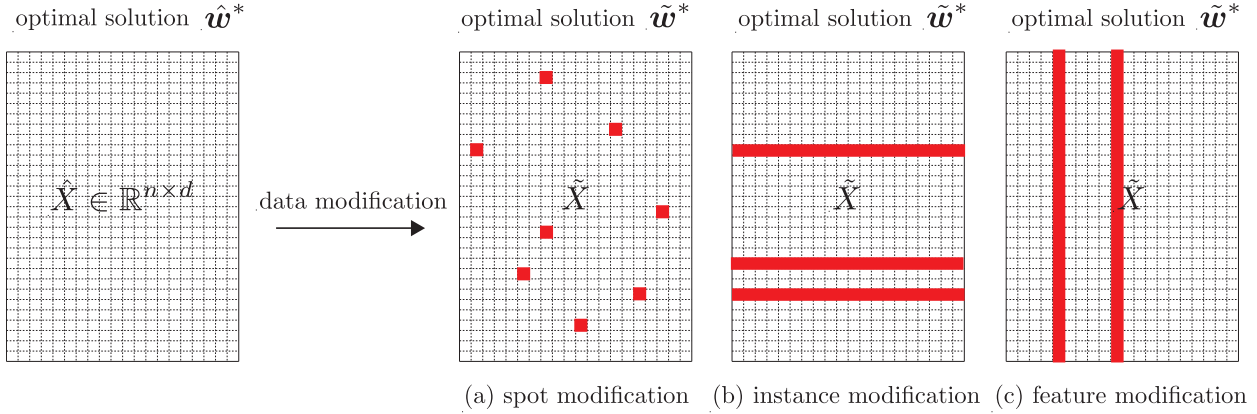}
 \end{center}
 \caption{Schematic illustration of the three scenarios considered in this paper. When a small part of the data matrix $X$ is modified, we need to quickly incorporate the effect of the data modification into the classifier. Our proposed method can compute a lower and an upper bound of the new optimal solution parameter $\tilde{w}_j^*$ with the cost only proportional to the number of modified elements.}
 \label{fig:illustration}
\end{figure}

Another novel contribution in this paper is to present methods for improving the bounds in \eqref{eq:beta-bounds} by what we call the \emph{partial optimization}.
When the number of features $d$ is large, optimizing all the $d$ coefficients $\{{w}_j\}_{j=1}^d$ would be very costly. 
For example, when $\cM = \{\{(i, j^\prime)\}_{i=1}^n\}$, i.e., the $j^\prime$$^{\rm th}$ column of $X$ is modified (scenario (c) in \figurename~\ref{fig:illustration}),
we would conjecture that
$j^\prime$$^{\rm th}$
parameter would change a lot
while other parameters would not change so much (unless the correlation of the feature with the $j^{\rm th}$ feature is very large). 
%
%In order to work with such a situation,
Thus we propose a partial optimization approach in which
%we do not optimize all the parameters;
we only optimize a part of the parameters that are expected to change largely.
A nice thing about the partial optimization approach is that the bounds in \eqref{eq:beta-bounds} can be improved with less computational cost than optimizing all the parameters.
In the above example, we would optimize only $j^\prime$$^{\rm th}$ parameter for computing tighter bounds in the form of \eqref{eq:beta-bounds}.
We also consider dual problems of linear classifier training where we have $n$ dual parameters $\{ \alpha_i \}_{i=1}^n$.
In the case where $i^\prime$$^{\rm th}$ row of the data matrix $X$ is modified (scenario (b) in \figurename~\ref{fig:illustration}), we can conduct partial optimization where only $\alpha_{i^\prime}$ is optimized while other dual parameters are fixed.
%because we would conjecture that $\alpha_{i^\prime}$ would change a lot after the data modification.
%
Although partial optimization-like approaches might have been used in practice as a simple heuristic, to the best of our knowledge, there are no previous studies which points out that partial optimization can be useful for theoretically-sound non-heuristic decision making. 

\subsection{Related Works}

In many practical machine learning tasks, we often need to solve multiple related optimization problems.
In such a case, \emph{warm-start} approach is sometimes useful where the optimal solution of a related problem is used as an initial starting point of the optimization problem~\cite{Decoste00,Cauwenberghs2001}.
In the data modification scenarios we consider here, we can use the old parameters $\hat{\bm {w}}^*$ as an initial starting point of the optimization problem for obtaining the new parameters $\tilde{\bm {w}}^*$.
Unfortunately, however, even when we use warm-start, the computational cost of re-training an SVM-like classifier is at least $\cO(nd)$ because we need to go through the entire data matrix $X$ at least once in the optimization algorithm.

In optimization lieterature, so-called \emph{sensitivity analysis} has been studied for the purpose of investigating the effect of \emph{problem parameters} on the optimal solution, where problem parameters refer to the parameters that define the optimization problem. 
In the data modification scenarios we consider here, each element of the data matrix $X$ can be considered as a problem parameter because it defines the optimization problem. 
Therefore, sensitivity analysis can be used for knowing how the classifier would change by the data modification in the vicinity of the old solution $\hat{\bm {w}}$. 
Unfortunately, however, sensitivity analysis cannot guarantee anything about the yet-to-be trained new solution $\tilde{\bm {w}}$.

The idea of bounding the solution of an optimization problem is inspired from the recent series of works on \emph{safe screening}~\cite{ghaoui2012safe,xiang2011learning,wang2013lasso,bonnefoy2014dynamic,liu2014safe,wang2014safe,xiang2014screening,fercoq2015mind,ndiaye2015gap}.
Safe screening is used for identifying non-active features or/and instances in sparse modeling. 
For example, in Lasso problem, activeness of each feature (whether the feature has nonzero coefficient or not) is characterized by an inequality constraint on the dual variables. 
A key idea in safe screening is that a part of non-active features can be identified without actually solving the optimization problem if one can compute a bound of the dual optimal solution. 
In deriving the bounds in the form of \eqref{eq:beta-bounds}, we use some techniques developed in safe screening literature. 

In the context of incremental learning, a closely related approach was recently proposed in \cite{okumura2015quick}, in which the authors considered the case where a small number of instances are incrementally added or/and deleted from $X$ (similar to scenario (b) in \figurename\ref{fig:illustration}), and proposed an algorithm for computing the bounds of the new solution with the cost depending on $\Delta n \times d$, where $\Delta n$ is the number of added and deleted instances.
Although their bound formula is totally different from ours, in the scenario on instance modification, the time complexities for computing their bounds and our bounds are same.
However, the method in \cite{okumura2015quick} cannot be used for other scenarios such as scenario (a) or (c) in \figurename\ref{fig:illustration}. 
Furthermore, as we clarify in \S\ref{sec:main}, our method can be applied to wider class of learning problems than the one considered in \cite{okumura2015quick}. 
The work in \cite{gabel2015monitoring} also studied similar problems in distributed stream learning context.
In this context, the authors in \cite{gabel2015monitoring} studied how to efficiently compute an upper bound of the difference between the old and the current solutions.

To the best of our knowledge, there are no other existing studies that study the above problem setup in this general form.
In the current big data era, it is quite important to be able to quickly incorporate the effect of a data modification in a huge database without going through the entire database.

%%%%%%%%%%%%%%%%%%%%%%%%%
%%%%%%%%%%%%%%%%%%%%%%%%%
\section{Problem Setup and Background} \label{sec:setup}
%%%%%%%%%%%%%%%%%%%%%%%%%
%%%%%%%%%%%%%%%%%%%%%%%%%

\subsection{Notations}
For any natural number $n$,
we define
$[n] := \{1, \ldots, n\}$.
We write the subdifferential operator as
$\partial$.
For a function $f$,
its domain is denoted as
${\rm dom}f$,
while its conjugate function is written as 
$f^*$. 
The
$L_2$ norm
of a vector
$\bm v \in \RR^m$
is
written as
$\|\bm v\|_2 := (\sum_{k \in [m]} |v_k|^2)^{1/2}$.

\subsection{Convex Regularized Learning Problems}
In this paper,
we study binary classification problems with $n$ training instances and $d$ features. 
The entire dataset is denoted as 
$(X, \bm y)$,
where 
$X \in \RR^{n \times d}$
is the input matrix
and
$\bm y \in \{\pm 1\}^n$
is the label vector. 
The $i^{\rm th}$ row vector,
the $j^{\rm th}$ column vector, 
and
$(i, j)^{\rm th}$ element of 
$X$
are respectively written as 
$\bm x_{i \cdot} \in \RR^d$,
$\bm x_{\cdot j} \in \RR^n$,
and 
$x_{ij} \in \RR$. 
Furthermore, 
we define
$Z := {\rm diag}(\bm y) X \in \RR^{n \times d}$
for notational simplicity,
and
$\bm z_{i \cdot} \in \RR^d$,
$\bm z_{\cdot j} \in \RR^n$,
$z_{ij} \in \RR$
are also defined similarly. 
We consider the following class of convex regularized empirical risk minimization problems:
\begin{align}
 \label{eq:primal}
 \bm {w}^* = \arg \min_{\bm {w} \in \RR^d} \cP(\bm {w}) := \frac{1}{n} \sum_{i \in [n]} \phi(\bm z_{i \cdot}^\top \bm {w}) + \psi(\bm {w}), 
\end{align}
where
$\phi: \RR \to \RR_+$
is a convex loss function 
and 
$\psi: \RR^d \to \RR_+$
is a convex penalty function.
Using Fenchel's duality theorem (see, e.g., Corollary 31.2.1 in \cite{rockafellar1970convex}), 
the dual problem of \eqref{eq:primal} is written as
\begin{align}
 \label{eq:dual}
 \bm \alpha^* := \arg \max_{\bm \alpha \in \RR^n} \cD(\bm \alpha) := - \frac{1}{n} \sum_{i \in [n]} \phi^*(-\alpha_i) - \psi^* \left(\frac{1}{n} Z^\top \bm \alpha \right),
\end{align}
where
$\phi^*$
and 
$\psi^*$
are convex conjugates of 
$\phi$ 
and
$\psi$,
respectively. 

\paragraph{Example (Smoothed-hinge SVM)}
As a working example, we study smoothed hinge SVM which is obtained by setting 
\begin{align}
 \label{eq:smoothed-hinge1}
 \phi(r) := 
\begin{cases}
 0 & ( r > 1 ), \\
 1 - r - \frac{\gamma}{2}  & (r < 1-\gamma), \\
 \frac{1}{2\gamma} (1 - r)^2 & ({\rm otherwise}),
\end{cases}
 ~~~~~~~~~
 \psi(\bm {w}) := \frac{\lambda}{2} \|\bm {w}\|_2^2, 
\end{align}
where $\gamma > 0$ and $\lambda > 0$ are tuning parameters.
Their convex conjugate functions are respectively written as
\begin{align}
 \label{eq:smoothed-hinge2}
 \phi^*(r) =
 \begin{cases}
  r + \frac{\gamma}{2} r^2 & (-1 \le r \le 0), \\
  +\infty & ({\rm otherwise}),
 \end{cases}
 ~~~~~~
 \psi^*(\bm {w}) = \frac{1}{2 \lambda} \|\bm {w}\|_2^2. 
\end{align}
%
%In \figurename~\ref{fig:xxx}, 
%the smoothed hinge loss function
%$\phi$
%is plotted. 

\subsection{Small Data Modification}
In this study,
we consider a situation that a small portion of the input matrix $X$ is modified. 
The set of modified elements in $X$ is denoted as 
$\cM \subseteq \{[n] \times [d]\}$
and its size as
$|\cM|$. 
Furthermore,
we define
$\cM_i := \{i \in [n] \mid \exists j \in [d] \text{ s.t. } (i, j) \in \cM\}$
and 
$\cM_j := \{j \in [d] \mid \exists i \in [n] \text{ s.t. } (i, j) \in \cM\}$. 
%
%\begin{center}
%$\cM(\cdot, \cJ) := \{i \in [n] \mid \exists j \in \cJ \text{ s.t. } (i, j) \in \cM\}$. 
%\end{center}
%
%Similarly,
%for a subset $\cI \subseteq [n]$, 
%we define
%\begin{center}
%$\cM(\cI, \cdot) := \{j \in [d] \mid \exists i \in \cI \text{ s.t. } (i, j) \in \cM\}$. 
%\end{center}
%
We call the problems before and after the data modification as \emph{old} problem and \emph{new} problem, respectively.
We use hat $\hat{~}$ notation for the old problem and tilde $\tilde{~}$ notation for the new problem;
e.g., 
$\hat{X}$,
$\hat{\bm x}_{i \cdot}$,
$\hat{\bm x}_{\cdot j}$,
$\hat{x}_{ij}$,
$\hat{Z}$,
$\hat{\bm z}_{i \cdot}$,
$\hat{\bm z}_{\cdot j}$,
$\hat{z}_{ij}$,
etc.
represent the old data before the modification,
while
$\tilde{X}$,
$\tilde{\bm x}_{i \cdot}$,
$\tilde{\bm x}_{\cdot j}$,
$\tilde{x}_{ij}$,
$\tilde{Z}$,
$\tilde{\bm z}_{i \cdot}$,
$\tilde{\bm z}_{\cdot j}$,
$\tilde{z}_{ij}$,
etc.
represent the new data after the modification. 
Furthermore, 
the primal and the dual optimal solutions for the old and the new problems are
respectively denoted as
$(\hat{\bm {w}}^*, \hat{\bm \alpha}^*)$
and 
$(\tilde{\bm {w}}^*, \tilde{\bm \alpha}^*)$. 

Our main contribution in this paper is to present methods for computing bounds of each primal variable
$\tilde{{w}}^*_j, j \in [d]$ 
and
dual variable 
$\tilde{\alpha}^*_i, i \in [n]$ 
with the computational cost depending only on the number of the modified elements $|\cM|$. 
A lower bound and an upper bound of a scalar quantity 
are respectively represented by using a capital letter $L$  and a capital letter $U$ 
such as 
\begin{align*}
 L(\tilde{{w}}^*_j) \le \tilde{{w}}^*_j \le U(\tilde{{w}}^*_j), 
 j \in [d],
 ~~~~~~
 L(\tilde{\alpha}^*_i) \le \tilde{\alpha}^*_i \le U(\tilde{\alpha}^*_i), 
 i \in [n].
\end{align*}

\subsection{Decision Making by New Solution's Bounds}
Using these bounds,
we can perform several decision making tasks on 
the new solution
$(\tilde{\bm {w}}^*, \tilde{\bm \alpha}^*)$. 

First, consider a situation that we want to classify a test instance $\bm x^\prime \in \RR^d$ based on the new classifier.
In this task,
we compute a lower bound and an upper bound of the classification score 
$f(\bm x^\prime) = \bm x^{\prime \top} \tilde{\bm {w}}^*$. 
Using the lower and the upper bounds,
the class label $y^\prime$ of the test instance $\bm x^\prime$ can be simply determined as 
\begin{align}
 \label{eq:test-classification}
 y^\prime = \text{sign}(\bm x^{\prime \top} \tilde{\bm {w}}^*) = 
 \begin{cases}
  -1 & (U(\bm x^{\prime \top} \tilde{\bm {w}}^*) < 0), \\
  +1 & (L(\bm x^{\prime \top} \tilde{\bm {w}}^*) \ge 0), \\  
  \text{unknown} & \text{(otherwise)}. 
 \end{cases}
\end{align}
Equation \eqref{eq:test-classification} suggests that, 
even if the new optimal solution 
$\tilde{\bm {w}}^*$
is not available,
if we have its sufficiently tight bounds, 
the classification task can be completed for some instances\footnote{
Using $L(\tilde{{w}}_j)$ and $U(\tilde{{w}}_j)$ for $j \in [d]$, we can compute a lower bound and an upper bound of $\bm x^{\prime \top} \tilde{\bm {w}}$ as 
%$L(\bm x^{\prime \top} \tilde{\bm {w}}^*)= \sum_{j \mid x_{ij} \ge 0} x_{ij} L(\tilde{{w}}_j^*) + \sum_{j \mid x_{ij} < 0} x_{ij} U(\tilde{{w}}_j^*)$
%and
%$U(\bm x^{\prime \top} \tilde{\bm {w}}^*) = \sum_{j \mid x_{ij} \ge 0} x_{ij} U(\tilde{{w}}_j^*) + \sum_{j \mid x_{ij} < 0} x_{ij} L(\tilde{{w}}_j^*)$,
%respectively.
\begin{align*}
 L(\bm x^{\prime \top} \tilde{\bm {w}}^*)
 =
 \displaystyle\sum_{j \mid x^\prime_{ij} \ge 0} x^\prime_{ij} L(\tilde{{w}}_j^*)
 + 
 \sum_{j \mid x^\prime_{ij} < 0} x^\prime_{ij} U(\tilde{{w}}_j^*),
 U(\bm x^{\prime \top} \tilde{\bm {w}}^*)
 =
 \sum_{j \mid x^\prime_{ij} \ge 0} x^\prime_{ij} U(\tilde{{w}}_j^*)
 + 
 \sum_{j \mid x^\prime_{ij} < 0} x^\prime_{ij} L(\tilde{{w}}_j^*).
\end{align*}
}.
Similar discussion has been done in \cite{okumura2015quick} and \cite{shibagaki2015regularization}. 

Next, consider a situation that the data matrix $X$ is constantly changing. 
In such a situation, an important decision making task is to determine when we should actually re-train the classifier.
Suppose that we have a tolerance threshold $\theta > 0$ for how much the classifier can be different from the optimal one.
If we quantify the difference by $L_2$ norm $\|\hat{\bm {w}}^* - \tilde{\bm {w}}^*\|_2$,
we can re-train the classifier only when
\begin{align}
 \label{eq:retrain-criterion}
 U(\| \hat{\bm {w}}^* - \tilde{\bm {w}}^* \|_2) \ge \theta. 
\end{align}
Equation \eqref{eq:retrain-criterion} indicates that,
even if the optimal solution 
$\tilde{\bm {w}}^*$
is not available,
we can make sure that the difference of the currently available solution $\hat{\bm {w}}^*$ from the unknown optimal solution $\tilde{\bm {w}}^*$ is no greater than our tolerance threshold\footnote{
Using $L(\tilde{{w}}_j)$ and $U(\tilde{{w}}_j)$ for $j \in [d]$, we can compute an upper bound of $\| \hat{\bm {w}} - \tilde{\bm {w}} \|_2$ as 
\begin{align*}
U(\| \hat{\bm {w}} - \tilde{\bm {w}} \|_2)
=
\sqrt{\displaystyle\sum_{j \in [d]}
\max\{\hat{{w}}^*_j - L(\tilde{{w}}^*_j), U(\tilde{{w}}^*_j) - \hat{{w}}^*_j\}^2}.
\end{align*}
}.
Similar discussion has been done in \cite{okumura2015quick} and \cite{gabel2015monitoring}. 

Finally, let us discuss the connection with safe screening \cite{ghaoui2012safe,xiang2011learning,wang2013lasso,bonnefoy2014dynamic,liu2014safe,wang2014safe,xiang2014screening,fercoq2015mind,ndiaye2015gap}. 
In the example of smoothed hinge SVM (see \eqref{eq:smoothed-hinge1} and \eqref{eq:smoothed-hinge2}),
the optimality condition indicates that 
$1 < \tilde{\bm z}_i^\top \tilde{\bm {w}}^* ~\Rightarrow~ \tilde{\alpha}^*_i = 0$.
This  condition indicates that
$L(\tilde{\bm z}_i^\top \tilde{\bm {w}}^*) > 1 ~\Rightarrow~ \tilde{\alpha}^*_i = 0$,
which means that the $i^{\rm th}$ instance would be non-support vector, and we can exclude the instance before solving the new optimization problem. 
This approach is referred to as \emph{safe sample screening}\cite{ogawa2013safe,wang2013scaling,Zimmert2015}. 
%%%%%%%%%%%%%%%%%%%%%%%%%
%%%%%%%%%%%%%%%%%%%%%%%%%
\section{Efficient Bound Computation after Small Data Change} \label{sec:main}
%%%%%%%%%%%%%%%%%%%%%%%%%
%%%%%%%%%%%%%%%%%%%%%%%%%

We present our main results in this section.
Our goal is to efficiently compute lower and upper bounds of the new primal-dual optimal solutions 
$\tilde{\bm {w}}^*$
and
$\tilde{\bm \alpha}^*$
by using the old primal-dual optimal solutions 
$(\hat{\bm {w}}^*, \hat{\bm \alpha}^*)$. 
If 
the loss function
$\phi$
or/and
the penalty function
$\psi$
have certain properties, 
we show that lower and upper bounds of 
$\tilde{{w}}^*_j, j \in [d]$
and
$\tilde{\alpha}^*_i, i \in [n]$
can be computed
with time complexity 
$\cO(|\cM|)$
where 
$|\cM|$
is the number of modified elements. 
When the number of the modified elements 
is much smaller
than
the number of the entire elements
$nd$,
it is quite beneficial to be able to compute the bounds of the new optimal solution in such an efficient way.

Before presenting the main theorem,
we summarize the properties of the loss function and the penalty function:

\noindent
{\bf Property A} (for loss function):
The loss function
$\phi$
is $\gamma^{-1}$-smooth,
i.e.,
$\phi$ is differentiable and its gradient $\partial \phi$ satisfies 
 \begin{align*}
	|\partial \phi(a) - \partial \phi(b)| \le \gamma^{-1} |a - b|
	~~~
	\forall a, b \in \RR.
\end{align*}

\noindent
{\bf Property B} (for penalty function):
       The penalty function 
       $\psi$
       is $\lambda$-strongly convex,
       i.e.,
       $\psi$ is sub-differentiable and its sub-gradient $\partial \psi$ satisfies 
       \begin{align*}
	\psi(\bm a) \ge \psi(\bm b) + \partial \psi(\bm b)^\top (\bm a - \bm b) + \frac{\lambda}{2} \|\bm a - \bm b\|_2^2
	~~~
	\forall \bm a, \bm b \in \RR^d.
       \end{align*}

At the end of this section,
we will present examples of the loss functions and the penalty functions that possess these properties. 
The following theorem tells that,
when the loss function and/or the penalty function have these properties, 
lower and upper bounds of the new solutions 
$\tilde{{w}}^*_j, j \in [d]$, 
and 
$\tilde{\alpha}^*_i, i \in [n]$
can be efficiently computed. 
% ------------------------------------------------------------------------------------------
% Theorem1
% ------------------------------------------------------------------------------------------
\begin{theorem}
 \label{theo:main}
 Assume that the following quantities
 \begin{align}
  \label{eq:memory-quantities}
  \hat{\bm z}_{i \cdot}^\top \hat{\bm {w}}, 
  ~
  \|\hat{\bm x}_{i \cdot}\|_2,
 ~
 \hat{\bm z}_{\cdot j}^\top \hat{\bm \alpha},
 ~
 \|\hat{\bm x}_{\cdot j}\|_2,
 ~
 \inf_{\bm \alpha \in {\rm dom} D(\bm \alpha)} n^{-1} \hat{\bm z}_{\cdot j}^\top \bm \alpha,
 ~
 \sup_{\bm \alpha \in {\rm dom} D(\bm \alpha)} n^{-1} \hat{\bm z}_{\cdot j}^\top \bm \alpha, 
 \end{align}
 for all $i \in [n]$ and $j \in [d]$, 
 are stored in the memory with the size $\cO(n + d)$,
 and the penalty function $\psi$ is \emph{decomposable}
 in the sense that there exists $d$ convex functions
 $\psi_j: \RR \to \RR$, $j \in [d]$,
 such that
 \begin{align*}
  \psi(\bm {w}) = \sum_{j \in [d]} \psi_j({w}_j)
  \text{ for all }
  \bm {w} \in \RR^d.
 \end{align*}
 Then,
 in each of the following three cases (i), (ii), and (iii), 
 the lower and the upper bounds of 
 $\{\tilde{{w}}^*_j\}_{j \in [d]}$
 and 
 $\{\tilde{\alpha}^*_i\}_{i \in [n]}$
 can be evaluated with time complexity $\cO(|\cM|)$. 

  Case (i) The loss function has {\bf Property A}: 
 \begin{align*}
  %\label{eq:beta-LB-D}
  \tilde{{w}}^*_j
  \ge 
  L_D(\tilde{{w}}^*_j)
  &:=
  \inf \partial \psi^*_j \left(
  \max\left\{
  \inf_{\bm \alpha \in {\rm dom} \cD(\bm \alpha)} n^{-1} \tilde{\bm z}_{\cdot j}^\top \bm \alpha, n^{-1} \tilde{\bm z}_j^\top \hat{\bm \alpha}^* - \|\tilde{\bm x}_{\cdot j}\|_2 \sqrt{2 \gamma^{-1} \tilde{G}(\hat{\bm {w}}^*, \hat{\bm \alpha}^*)}
  \right\}
  \right),
  \\
  %\label{eq:beta-UB-D}
  \tilde{{w}}^*_j
  \le 
  U_D(\tilde{{w}}^*_j)
  &:=
  \sup \partial \psi^*_j \left(
  \min\left\{
  \sup_{\bm \alpha \in {\rm dom} \cD(\bm \alpha)} n^{-1} \tilde{\bm z}_{\cdot j}^\top \bm \alpha, n^{-1} \tilde{\bm z}_j^\top \hat{\bm \alpha}^* + \|\tilde{\bm x}_{\cdot j}\|_2 \sqrt{2 \gamma^{-1} \tilde{G}(\hat{\bm {w}}^*, \hat{\bm \alpha}^*)}
  \right\}
  \right), 
  \\
  \tilde{\alpha}_i^*
  \ge 
  L_D(\tilde{\alpha}^*_i)
  &:=
  \min\left\{
  \inf_{\bm \alpha \in {\rm dom} \cD(\bm \alpha)} \alpha_i, 
  \hat{\alpha}_i - \sqrt{2n \gamma^{-1} \tilde{G}(\hat{\bm {w}}^*, \hat{\bm \alpha}^*)}
  \right\},
  \\
  %\label{eq:alpha-UB-D}
  \tilde{\alpha}_i^*
  \le 
  L_D(\tilde{\alpha}^*_i)
  &:=
  \max\left\{
  \sup_{\bm \alpha \in {\rm dom} \cD(\bm \alpha)} \alpha_i, 
  \hat{\alpha}_i + \sqrt{2n \gamma^{-1} \tilde{G}(\hat{\bm {w}}^*, \hat{\bm \alpha}^*)}
  \right\},
 \end{align*}
 where
 \begin{align*}
  %\label{eq:duality-gap}
  %\!\!\!\!\! \!\!
  \tilde{G}(\hat{\bm {w}}^*, \hat{\bm \alpha}^*)
  :=
  n^{-1}
  \sum_{i \in \cM_i}
  \left(
  \phi(\tilde{\bm z}_{i \cdot}^\top \hat{\bm {w}}^*)
  -
  \phi(\hat{\bm z}_{i \cdot}^\top \hat{\bm {w}}^*)
  \right)
  +
  \sum_{j \in \cM_j}
  \left(
  \psi^*_j(n^{-1}\tilde{\bm z}_{\cdot j}^\top \hat{\bm \alpha}^*)
  - 
  \psi^*_j(n^{-1}\hat{\bm z}_{\cdot j}^\top \hat{\bm \alpha}^*)
  \right).
 \end{align*}
 
 Case (ii) The penalty function has {\bf Property B}:
 \begin{align*}
  %\label{eq:beta-LB-P}
  \tilde{{w}}^*_j
  \ge
  L_P(\tilde{{w}}^*_j)
  &:=
  \hat{{w}}_j - \sqrt{2 \lambda^{-1} \tilde{G}(\hat{\bm {w}}^*, \hat{\bm \alpha}^*)},
  ~~~
  %\label{eq:beta-UB-P}
  \tilde{{w}}^*_j
  \le
  U_P(\tilde{{w}}^*_j)
  &:=
  \hat{{w}}_j + \sqrt{2 \lambda^{-1} \tilde{G}(\hat{\bm {w}}^*, \hat{\bm \alpha}^*)},
  \end{align*}
 \vspace*{-10mm}
 \begin{align*}
  %\label{eq:alpha-LB-P}
  \\
  \tilde{\alpha}_i^*
  \ge 
  L_P(\tilde{\alpha}^*_i)
  &:=
  \inf \left(- \partial \phi \left(
  \tilde{\bm z}_{i \cdot}^\top \hat{\bm {w}}^* + \|\tilde{\bm x}_{\cdot i}\|_2 \sqrt{2 \lambda^{-1} \tilde{G}(\hat{\bm {w}}^*, \hat{\bm \alpha}^*)}
  \right)\right),
  \\
  %\label{eq:alpha-UB-P}
  \tilde{\alpha}_i^*
  \le
  U_P(\tilde{\alpha}^*_i)
  &:=
  \sup \left(- \partial \phi \left(
  \tilde{\bm z}_{i \cdot}^\top \hat{\bm {w}}^* - \|\tilde{\bm x}_{i \cdot}\|_2 \sqrt{2 \lambda^{-1} \tilde{G}(\hat{\bm {w}}^*, \hat{\bm \alpha}^*)}
  \right)\right).
 \end{align*}

 Case (iii) The loss function has {\bf Property A} and the penalty function has {\bf Property B}:
 \begin{align*}
  %\label{eq:beta-LB-PD}
  \tilde{{w}}^*_j
  \ge
  L_{PD}(\tilde{{w}}^*_j)
  :=
  \max\left\{
  L_{P}(\tilde{{w}}^*_j),
  L_{D}(\tilde{{w}}^*_j)
  \right\},
  ~~~
  %\label{eq:beta-UB-PD}
  \tilde{{w}}^*_j
  \le
  U_{PD}(\tilde{{w}}^*_j)
  :=
  \min\left\{
  U_{P}(\tilde{{w}}^*_j),
  U_{D}(\tilde{{w}}^*_j)
  \right\},
  \end{align*}
 \begin{align*}
  %\label{eq:alpha-LB-PD}
  \tilde{\alpha}^*_i
  \ge
  L_{PD}(\tilde{\alpha}^*_i)
  :=
  \max\left\{
  L_{P}(\tilde{\alpha}^*_i),
  L_{D}(\tilde{\alpha}^*_i)
  \right\},
  ~~~
  %\label{eq:alpha-UB-PD}
  \tilde{\alpha}^*_i
  \le
  U_{PD}(\tilde{\alpha}^*_i)
  :=
  \min\left\{
  U_{P}(\tilde{\alpha}^*_i),
  U_{D}(\tilde{\alpha}^*_i)
  \right\}.
 \end{align*}

% Furthermore,
% if all the quantities listed in {\bf Property D} are stored in the memory,
% all the bounds in
% \eqref{eq:beta-LB-D}
% to 
% \eqref{eq:alpha-UB-PD}
% can be computed with $\cO(|\cM|)$ time complexity. 
\end{theorem}
The proof of Theorem~\ref{theo:main} is presented in Appendix~\ref{app:proof}.
In the proof, we use a technique recently introduced in the context of safe screening based on the duality gap of an approximate solution~\cite{fercoq2015mind,ndiaye2015gap}.
Note that the tightness of the bounds in the theorem depends on a quantity 
$\tilde{G}(\hat{\bm {w}}^*, \hat{\bm \alpha}^*)$,
which indicates the duality gap of the new optimization problem evaluated at the old primal-dual solutions
$\hat{\bm {w}}^*$
and 
$\hat{\bm \alpha}^*$. 
For proving the time complexity in the theorem, 
we exploit a simple fact that the duality gap 
$\tilde{G}(\hat{\bm {w}}^*, \hat{\bm \alpha}^*)$
can be computed with time complexity 
$\cO(|\cM|)$\footnote{
Note that the duality gap of the old problem evaluated at 
$\hat{\bm {w}}^*$
and 
$\hat{\bm \alpha}^*$
is zero
because of their optimality. 
If the duality gap of the old problem evaluated at 
$\hat{\bm {w}}^*$
and 
$\hat{\bm \alpha}^*$
is not exactly zero due to numerical issue, 
we can simply add the remaining gap
to 
$\tilde{G}(\hat{\bm {w}}^*, \hat{\bm \alpha}^*)$. 
}.
The assumption that all the quantities in \eqref{eq:memory-quantities} are available in the memory is reasonable because 
they can be computed when the old optimization problem was solved for obtaining $\hat{\bm {w}}^*$ and $\hat{\bm \alpha}^*$. 
The assumption on decomposability is satisfied in many penalty functions including $L_q^q$-norm with $q \ge 1$. 

Many popular loss functions for binary classification problems such as squared hinge loss and logistic regression loss have {\bf Property A}. 
Penalty function that has {\bf Property B} includes $L_2$ penalty and elastic-net penalty, etc. 
Since we can efficiently compute bounds if either {\bf Property A} or {\bf Property B} is satisfied, 
we can, for example, compute bounds for standard vanilla-hinge SVM, $L_1$-penalized logistic regression, and many other popular classification problems. 
Similar bound computation methods can be derived also for a wide class of regularized convex regression problems. 

The bounds in Theorem \ref{theo:main} for smoothed-hinge SVM are presented in Appendix \ref{app:smoothe-hinge-SVM}.

%%%%%%%%%%%%%%%%%%%%%%%%%
%%%%%%%%%%%%%%%%%%%%%%%%%
\section{Partial Optimization for Obtaining Tighter Bounds} \label{sec:extensions}
%%%%%%%%%%%%%%%%%%%%%%%%%
%%%%%%%%%%%%%%%%%%%%%%%%%

Let us consider how to handle a situation where the bounds in Theorem~\ref{theo:main} are not sufficiently tight enough. 
If we do not mind spending computational cost depending on the entire data matrix size $nd$, the most naive solution would be to just re-train the classifier with the modified new dataset.
However, when $n$ and $d$ are very large, it would be desirable to be able to obtain tighter bounds than those in Theorem~\ref{theo:main} with a little more additional cost, but still not depending on $nd$. 

Our idea is to simply note that 
the tightness of the bounds
in Theorem~\ref{theo:main}
depends on the \emph{closeness} of the old solution
$(\hat{\bm {w}}^*, \hat{\bm \alpha}^*)$
to
the new solutions 
$(\tilde{\bm {w}}^*, \tilde{\bm \alpha}^*)$. 
It suggests that, if we could replace the old solution with another solution that is closer to the new solution, the bounds could be tightened. 
We propose a simple method for obtaining such a closer solution with the computational cost not depending on the entire data matrix size, and show that the method always improve the bounds in Theorem~\ref{theo:main}. 

A primal variable ${w}_j$, $j \in [d]$, would be highly dependent on the $j^{\rm th}$ column of $X$. 
It suggests that,
if many elements in the $j^{\rm th}$ column of $X$ are modified,
the $j^{\rm th}$ primal variable
${w}_j$
would change a lot after the data modification, 
i.e., 
the difference between
$\hat{{w}}^*_j$
and 
$\tilde{{w}}^*_j$
would be large. 
Similarly, 
a dual variable $\alpha_i$, $i \in [n]$, would be highly affected by the $i^{\rm th}$ row of $X$. 
If many elements in the $i^{\rm th}$ row of $X$ are modified,
the difference between
$\hat{\alpha}^*_i$
and 
$\tilde{\alpha}^*_i$
would be large. 

These observations suggest that
we might be able to obtain a solution closer to the new solution with reasonable computational cost
by only optimizing a part of the variables that are expected to be highly affected by the data modification, 
while other variables are fixed. 
If we consider solving the primal optimization problem only w.r.t. a subset of the primal variables ${w}_j, j \in \cJ \subseteq [d]$,
the computational cost depends only on $n |\cJ|$. 
Similarly,
if we consider solving the dual optimization problem only w.r.t. a subset of the dual variables $\alpha_i, i \in \cI \subseteq [n]$,
the computational cost depends only on $d |\cI|$. 
The following theorem tells that 
the bounds in Theorem~\ref{theo:main} can be strictly tightened 
if these partial optimization problems have a strict progress in terms of their objective function values. 
%
% ---------------------------------------------------------------------------
% Theorem for partial optimization
% ---------------------------------------------------------------------------
\begin{theorem}
\label{theo:partial-optimization}
Assume that the penalty function is decomposable as defined in Theorem~\ref{theo:main}. 
Consider
partially optimizing the primal problem
\eqref{eq:primal}
only w.r.t.
a subset of the primal variables 
$\{{w}_j\}_{j \in \cJ \subseteq [d]}$. 
Let
$\check{\bm {w}}^*$
be a $d$-dimensional vector 
whose elements corresponding to $\cJ$
are the solution of the partial primal problem, 
while the elements corresponding to $[d] \setminus \cJ$ 
are the old primal solutions 
$\{\hat{{w}}^*_j\}_{j \in [d] \setminus \cJ}$. 
Then,
unless 
$\check{\bm {w}}^* = \hat{\bm {w}}^*$, 
%the objective function of the partial primal optimization problem strictly decreases,
all the bounds in Theorem~\ref{theo:main} will be strictly tightened
by replacing
$\hat{\bm {w}}^*$
with
$\check{\bm {w}}^*$. 

Similarly, 
consider partially optimizing the dual problem
\eqref{eq:dual}
only w.r.t.
a subset of the dual variables 
$\{\alpha_i\}_{i \in \cI \subseteq [n]}$. 
Let 
$\check{\bm \alpha}^*$
be an $n$-dimensional vector
whose elements corresponding to $\cI$
are the solution of the partial dual problem, 
while the elements corresponding to $[n] \setminus \cI$ 
are the old dual solutions 
$\{\hat{\alpha}^*_i\}_{i \in [n] \setminus \cI}$. 
Then,
unless
$\check{\bm \alpha}^* = \hat{\bm \alpha}^*$,
%if the objective function of the partial dual optimization problem strictly increases, 
all the bounds in Theorem~\ref{theo:main} will be strictly tightened
by replacing
$\hat{\bm \alpha}^*$
with
$\check{\bm \alpha}^*$. 
\end{theorem}
The proof of Theorem~\ref{theo:partial-optimization} is presented in Appendix~\ref{app:proof}. 

For each of the three scenarios depicted in \figurename~\ref{fig:illustration}, we consider specific partial optimization strategy. 
First, for spot modification scenario, we set $\cJ := \cM_j$ and $\cI := \cM_i$, i.e., we optimized the primal (resp. dual) variables if there are at least one modification at the corresponding columns (resp. rows) of $X$. 
For instance modification scenario, we set $\cJ := \emptyset$ and $\cI := \cM_i$, i.e., we only optimized dual variables corresponding to the modified rows of $X$. 
Similarly for feature modification scenario, we set $\cJ := \cM_j$ and $\cI := \emptyset$, i.e., we only optimized primal variables corresponding to the modified columns of $X$.

%%%%%%%%%%%%%%%%%%%%%%%%%
%%%%%%%%%%%%%%%%%%%%%%%%%
\section{Experiment} \label{sec:exp}
%%%%%%%%%%%%%%%%%%%%%%%%%
%%%%%%%%%%%%%%%%%%%%%%%%%
In this section, we empirically investigate the tightness of the bounds obtained by the approaches 
in \S\ref{sec:main} and \S\ref{sec:extensions}. 
%
%%Due to space limitations, we present a part of our results in the main text (the complete experimental results are presented in Appendix~\ref{app:exp}).

\subsection{Settings}

We evaluated the tightness of the bounds in terms of the performances of the two tasks described in \S\ref{sec:setup}.
The first task is test instance classification,
where the predicted test label $y^\prime$ for a test instance $\bm x^\prime \in \RR^d$ can be determined by the bounds of the classification score $L(\bm x^{\prime \top} \tilde{\bm w}^*)$ and $U(\bm x^{\prime \top} \tilde{\bm w}^*)$ as in \eqref{eq:test-classification}.
To see the performances of this task, we report the rate of the test instances whose predicted class labels are determined.
The second task is to evaluate how much the solution can change by the data modification.
To see the performances of this task, we report the upper bound of the parameter change in terms of the $L_2$ norm.

\begin{table}[t]
\begin{center}
\begin{minipage}[c]{0.45\hsize}
\caption{Data sets used for the experiment}
\label{table:datasets}
\begin{itemize}
\item All are from LIBSVM data repository \cite{chang2011libsvm}.
\item 80\% of instances are used for training while others are for validation.
\item Every instance $\bm{x}$ is normalized to norm $\|\bm{x}\|_2 = 1$.
\end{itemize}
\end{minipage}
\begin{minipage}[c]{0.45\hsize}
\begin{tabular}{crr}
\hline
Name & \multicolumn{1}{c}{$n$} & \multicolumn{1}{c}{$d$} \\
\hline
{\tt kdd-a}      & 8,407,752 & 20,216,830 \\
{\tt url}        & 2,396,130 &  3,231,961 \\
{\tt news20}     &    19,996 &  1,355,191 \\
{\tt real-sim}   &    72,309 &     20,958 \\
{\tt rcv1-train} &    20,242 &     47,236 \\
\hline
\end{tabular}
\end{minipage}
\end{center}
\end{table}

We studied smoothed-hinge SVM in \eqref{eq:smoothed-hinge1} and \eqref{eq:smoothed-hinge2} on five benchmark datasets obtained from libsvm data repository~\cite{chang2011libsvm} (Table \ref{table:datasets}), where 80\% of the instances are used as training set and the rest are used as test set.
We considered three scenarios in \figurename~\ref{fig:illustration}.
In spot modification scenario, we modified randomly chosen $|\cM|$ elements in $X$ for $|\cM| \in \{1, 100, 10000\}$. 
In instance modification scenario, we randomly chose $\Delta n$ rows of $X$ for $\Delta n \in \{1, 10, 100\}$, and modified all non-zero values in them.
In feature modification scenario, we randomly chose $\Delta d$ columns of $X$ for $\Delta d \in \{1, 10, 100\}$, and modified all non-zero values in them.
To modify a value, we replaced it with a random value taken from the uniform distribution between the minimum and the maximum value of the corresponding feature in the data set.
We set $\lambda = \{0.001, 0.01, 0.1, 1\}$. 
For each condition, we run the experiment with 10 different random seeds.

\subsection{Results}

Figures \ref{fig:experiment-all-1} to \ref{fig:experiment-all-4} show the results of the label determination rate,
%for $\lambda = \{0.001, 0.01, 0.1, 1\}$,
Figures \ref{fig:experiment-all-5} to \ref{fig:experiment-all-8} show the upper bound of the parameter change,
and Table \ref{table:computation-time-all} show the ratio of the computation time of bound computation in \S\ref{sec:main} to re-training.

The red points in the figures show the performances of the bounds computed by our proposed method in \S\ref{sec:main}. 
For {\tt kdd-a} dataset, more than 99.9\% of the test labels were determined by using the bounds, and the upper bounds of the parameter change were also highly stable. 
We conjecture that, since this dataset is huge ($n > \text{8 million}$ and $d > \text{20 million}$), small amount of data modification did not change the solution much, and our bounds could nicely capture this phenomenon. 
For {\tt url} dataset, the label determination rates were slightly worse especially when 100 rows were modified in instance modification scenario. For smaller datasets, we can confirm that the rates become worse.
We conjecture that this performance deterioration might happen when many influential row vectors were modified and the change of the solution was relatively large.  
Table~\ref{table:computation-time-all} compares the computational costs of our bound computation method with those of re-training (with warm-start). 
From the table, we observe that the cost of our bound computation method is almost negligible.
% (remember that the computational cost for computing these bounds only depends on the number of modified elements in $X$, which is much smaller than the time complexity of re-training process, which takes at least $\cO(nd)$).

The blue points in the figures show the performances of the bounds lifted-up by partial optimization approach discussed in \S\ref{sec:extensions}. 
The black lines connecting red and blue points indicate the correspondence in 10 random trials. 
As described in Theorem~\ref{theo:partial-optimization}, the bounds in blue points are always tighter than those in red points 
(higher value in label determination rate and lower values in parameter change indicate tighter bounds). 
For {\tt kdd-a} dataset, the performance of the red and blue points are almost same because the original bounds in red points were already tight enough. 
For other datasets, the partial optimization approach seemed to work well.
Especially when the performances of the original bounds in red points were relatively poor (e.g., $\Delta n = 100$ case in instance modification scenario), the improvements by partial optimization were significant. 

\begin{figure}[p]
\caption{Label determination rate for $\lambda=0.001$}
\label{fig:experiment-all-1}
\begin{tabular}{ccc}
\hline
\multicolumn{3}{c}{Data set: {\tt kdda}}
\\ \hline
\begin{minipage}{0.32\hsize}
\begin{center}
Feature modification

\includegraphics[clip,width=5cm]{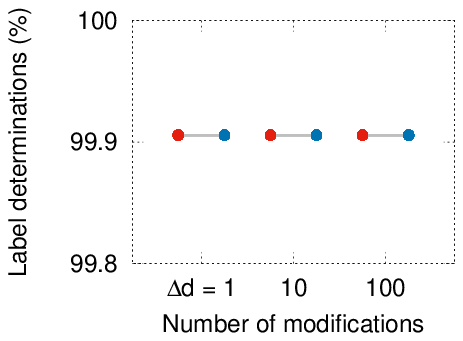}
\end{center}
\end{minipage}
&
\begin{minipage}{0.32\hsize}
\begin{center}
Instance modification

\includegraphics[clip,width=5cm]{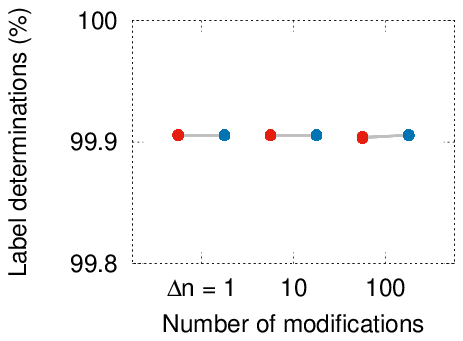}
\end{center}
\end{minipage}
&
\begin{minipage}{0.32\hsize}
\begin{center}
Spot modification

\includegraphics[clip,width=5cm]{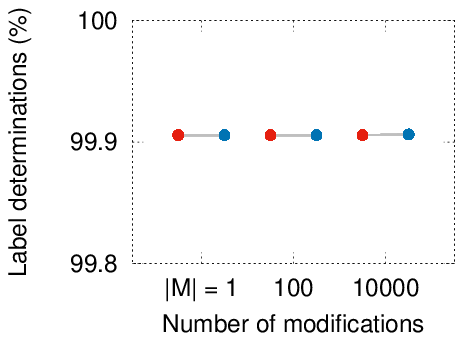}
\end{center}
\end{minipage}
\\ \hline
\multicolumn{3}{c}{Data set: {\tt url}}
\\ \hline
\begin{minipage}{0.32\hsize}
\begin{center}
Feature modification

\includegraphics[clip,width=5cm]{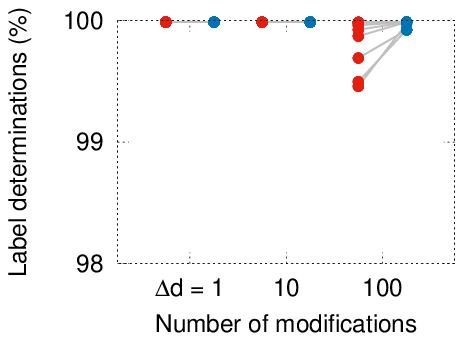}
\end{center}
\end{minipage}
&
\begin{minipage}{0.32\hsize}
\begin{center}
Instance modification

\includegraphics[clip,width=5cm]{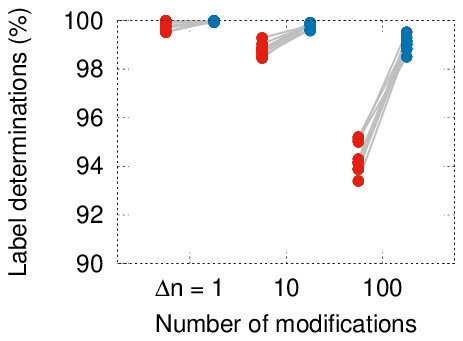}
\end{center}
\end{minipage}
&
\begin{minipage}{0.32\hsize}
\begin{center}
Spot modification

\includegraphics[clip,width=5cm]{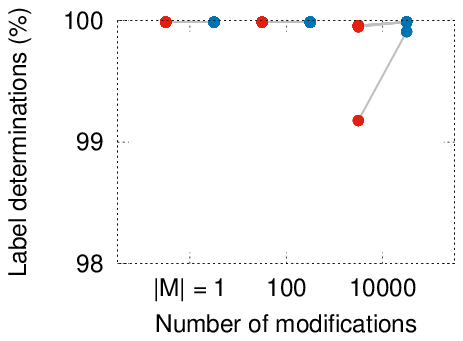}
\end{center}
\end{minipage}
\\ \hline
\multicolumn{3}{c}{Data set: {\tt news20}}
\\ \hline
\begin{minipage}{0.32\hsize}
\begin{center}
Feature modification

\includegraphics[clip,width=5cm]{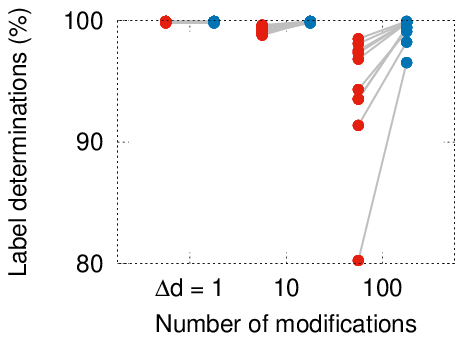}
\end{center}
\end{minipage}
&
\begin{minipage}{0.32\hsize}
\begin{center}
Instance modification

\includegraphics[clip,width=5cm]{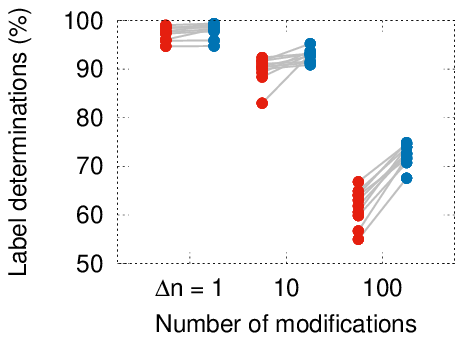}
\end{center}
\end{minipage}
&
\begin{minipage}{0.32\hsize}
\begin{center}
Spot modification

\includegraphics[clip,width=5cm]{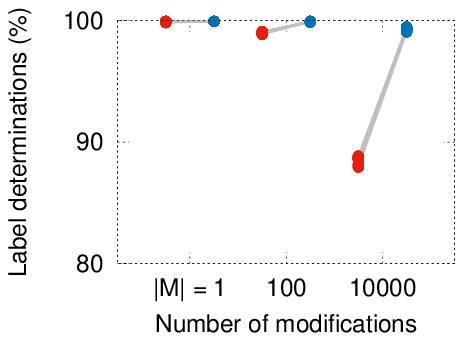}
\end{center}
\end{minipage}
\\ \hline
\multicolumn{3}{c}{Data set: {\tt real-sim}}
\\ \hline
\begin{minipage}{0.32\hsize}
\begin{center}
Feature modification

\includegraphics[clip,width=5cm]{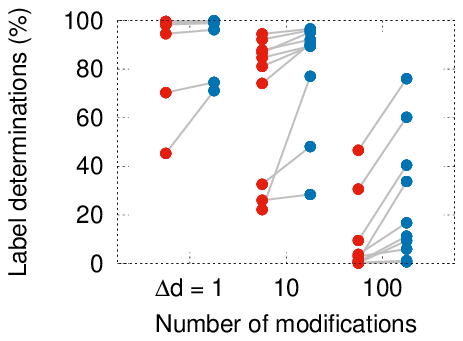}
\end{center}
\end{minipage}
&
\begin{minipage}{0.32\hsize}
\begin{center}
Instance modification

\includegraphics[clip,width=5cm]{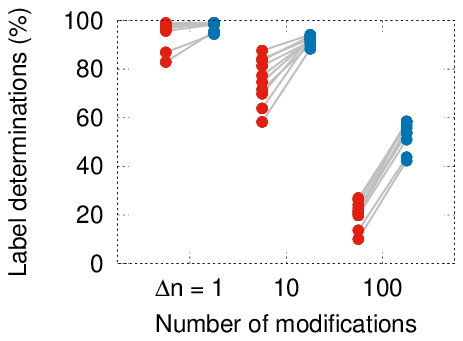}
\end{center}
\end{minipage}
&
\begin{minipage}{0.32\hsize}
\begin{center}
Spot modification

\includegraphics[clip,width=5cm]{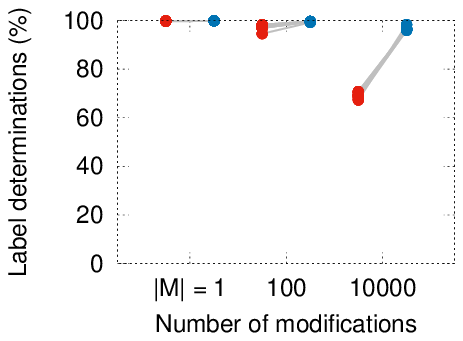}
\end{center}
\end{minipage}
\\ \hline
\multicolumn{3}{c}{Data set: {\tt rcv1-train}}
\\ \hline
\begin{minipage}{0.32\hsize}
\begin{center}
Feature modification

\includegraphics[clip,width=5cm]{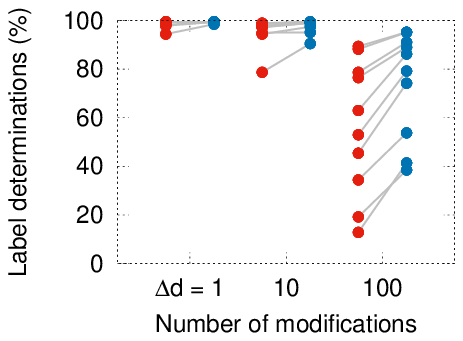}
\end{center}
\end{minipage}
&
\begin{minipage}{0.32\hsize}
\begin{center}
Instance modification

\includegraphics[clip,width=5cm]{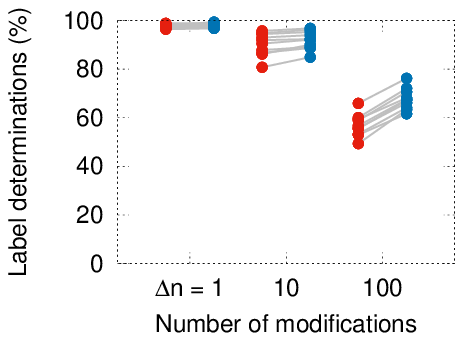}
\end{center}
\end{minipage}
&
\begin{minipage}{0.32\hsize}
\begin{center}
Spot modification

\includegraphics[clip,width=5cm]{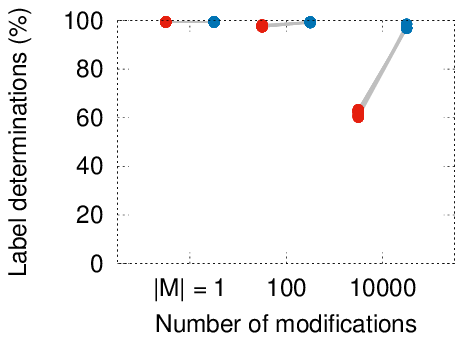}
\end{center}
\end{minipage}
\\ \hline
\end{tabular}
\end{figure}
\begin{figure}[p]
\caption{Label determination rate for $\lambda=0.01$}
\label{fig:experiment-all-2}
\begin{tabular}{ccc}
\hline
\multicolumn{3}{c}{Data set: {\tt kdda}}
\\ \hline
\begin{minipage}{0.32\hsize}
\begin{center}
Feature modification

\includegraphics[clip,width=5cm]{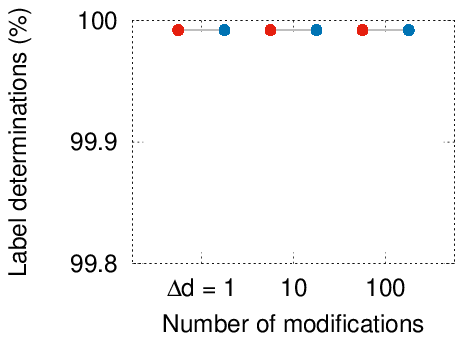}
\end{center}
\end{minipage}
&
\begin{minipage}{0.32\hsize}
\begin{center}
Instance modification

\includegraphics[clip,width=5cm]{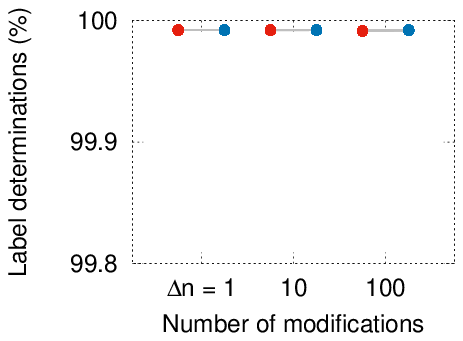}
\end{center}
\end{minipage}
&
\begin{minipage}{0.32\hsize}
\begin{center}
Spot modification

\includegraphics[clip,width=5cm]{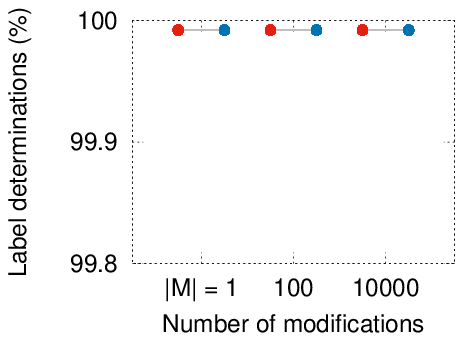}
\end{center}
\end{minipage}
\\ \hline
\multicolumn{3}{c}{Data set: {\tt url}}
\\ \hline
\begin{minipage}{0.32\hsize}
\begin{center}
Feature modification

\includegraphics[clip,width=5cm]{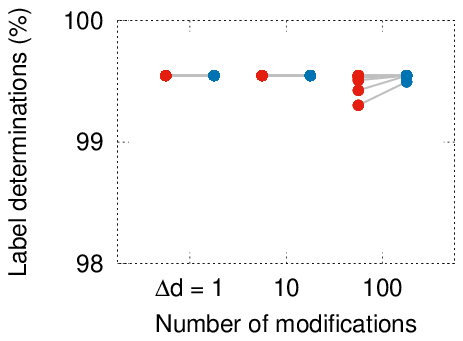}
\end{center}
\end{minipage}
&
\begin{minipage}{0.32\hsize}
\begin{center}
Instance modification

\includegraphics[clip,width=5cm]{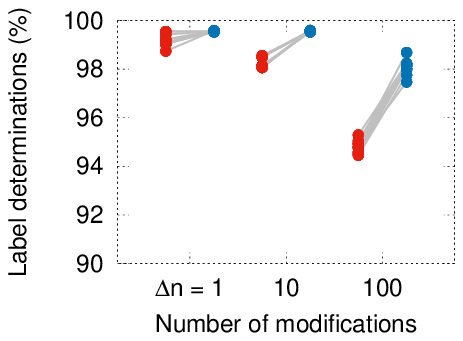}
\end{center}
\end{minipage}
&
\begin{minipage}{0.32\hsize}
\begin{center}
Spot modification

\includegraphics[clip,width=5cm]{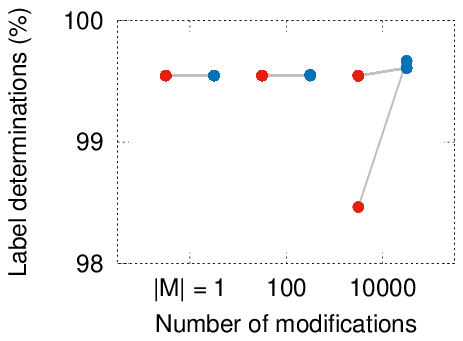}
\end{center}
\end{minipage}
\\ \hline
\multicolumn{3}{c}{Data set: {\tt news20}}
\\ \hline
\begin{minipage}{0.32\hsize}
\begin{center}
Feature modification

\includegraphics[clip,width=5cm]{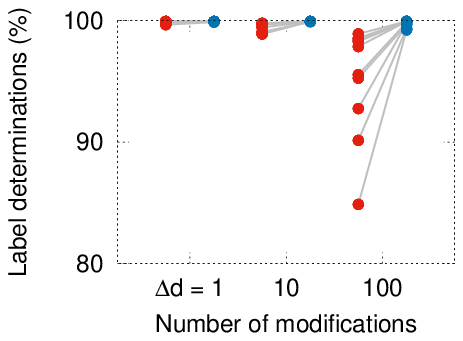}
\end{center}
\end{minipage}
&
\begin{minipage}{0.32\hsize}
\begin{center}
Instance modification

\includegraphics[clip,width=5cm]{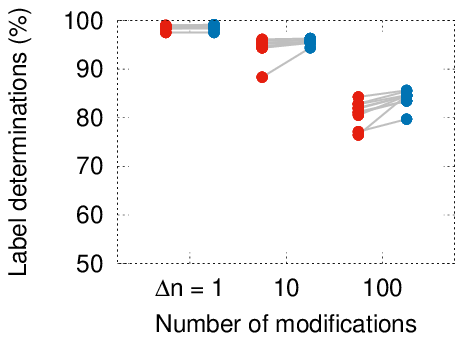}
\end{center}
\end{minipage}
&
\begin{minipage}{0.32\hsize}
\begin{center}
Spot modification

\includegraphics[clip,width=5cm]{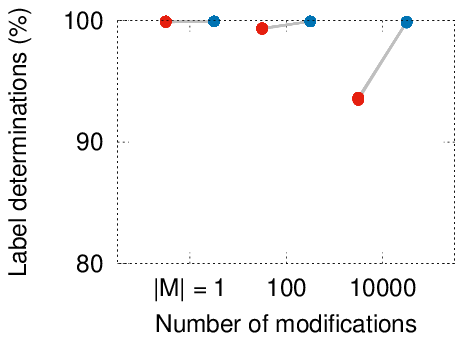}
\end{center}
\end{minipage}
\\ \hline
\multicolumn{3}{c}{Data set: {\tt real-sim}}
\\ \hline
\begin{minipage}{0.32\hsize}
\begin{center}
Feature modification

\includegraphics[clip,width=5cm]{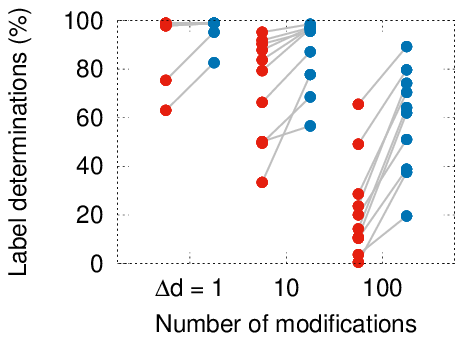}
\end{center}
\end{minipage}
&
\begin{minipage}{0.32\hsize}
\begin{center}
Instance modification

\includegraphics[clip,width=5cm]{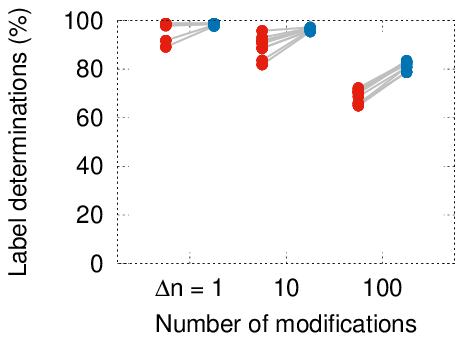}
\end{center}
\end{minipage}
&
\begin{minipage}{0.32\hsize}
\begin{center}
Spot modification

\includegraphics[clip,width=5cm]{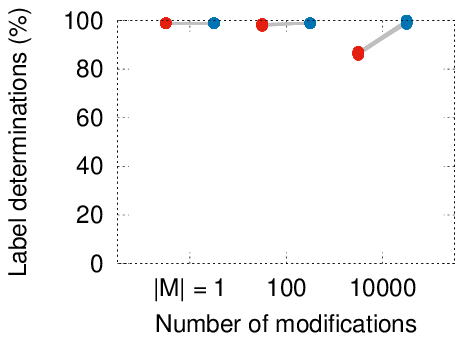}
\end{center}
\end{minipage}
\\ \hline
\multicolumn{3}{c}{Data set: {\tt rcv1-train}}
\\ \hline
\begin{minipage}{0.32\hsize}
\begin{center}
Feature modification

\includegraphics[clip,width=5cm]{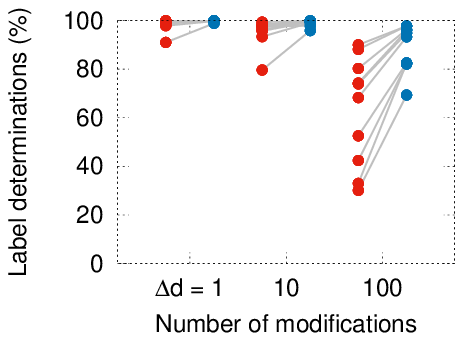}
\end{center}
\end{minipage}
&
\begin{minipage}{0.32\hsize}
\begin{center}
Instance modification

\includegraphics[clip,width=5cm]{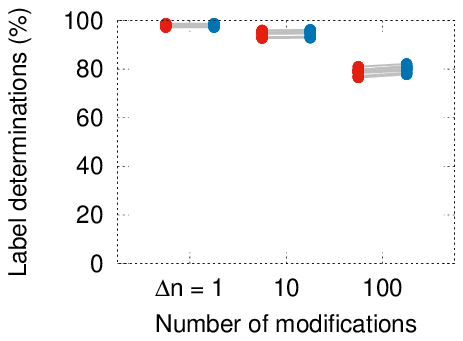}
\end{center}
\end{minipage}
&
\begin{minipage}{0.32\hsize}
\begin{center}
Spot modification

\includegraphics[clip,width=5cm]{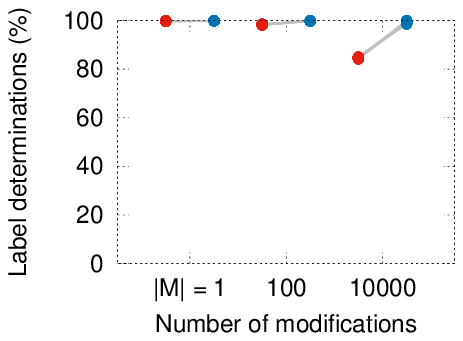}
\end{center}
\end{minipage}
\\ \hline
\end{tabular}
\end{figure}
\begin{figure}[p]
\caption{Label determination rate for $\lambda=0.1$}
\label{fig:experiment-all-3}
\begin{tabular}{ccc}
\hline
\multicolumn{3}{c}{Data set: {\tt kdda}}
\\ \hline
\begin{minipage}{0.32\hsize}
\begin{center}
Feature modification

\includegraphics[clip,width=5cm]{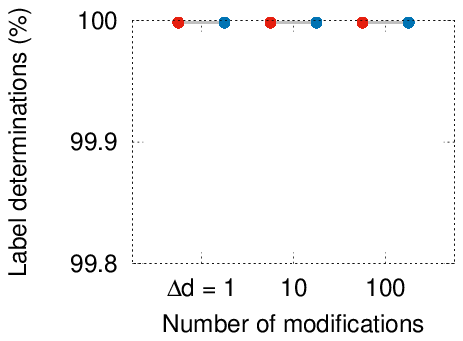}
\end{center}
\end{minipage}
&
\begin{minipage}{0.32\hsize}
\begin{center}
Instance modification

\includegraphics[clip,width=5cm]{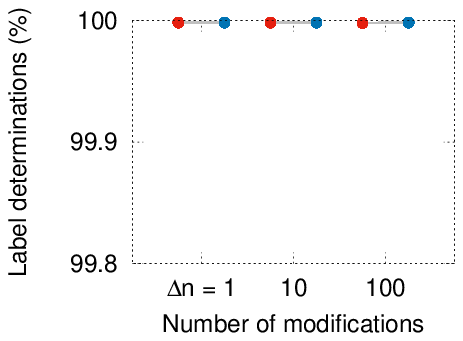}
\end{center}
\end{minipage}
&
\begin{minipage}{0.32\hsize}
\begin{center}
Spot modification

\includegraphics[clip,width=5cm]{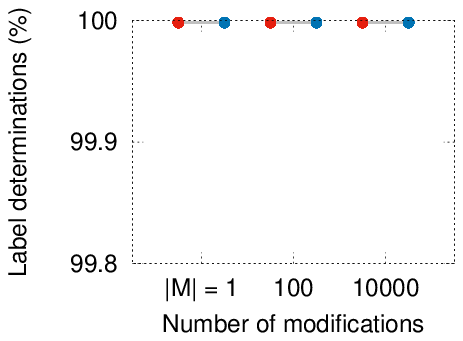}
\end{center}
\end{minipage}
\\ \hline
\multicolumn{3}{c}{Data set: {\tt url}}
\\ \hline
\begin{minipage}{0.32\hsize}
\begin{center}
Feature modification

\includegraphics[clip,width=5cm]{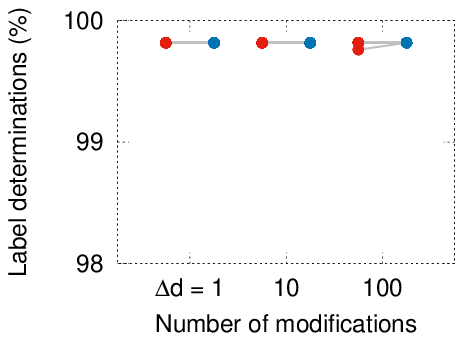}
\end{center}
\end{minipage}
&
\begin{minipage}{0.32\hsize}
\begin{center}
Instance modification

\includegraphics[clip,width=5cm]{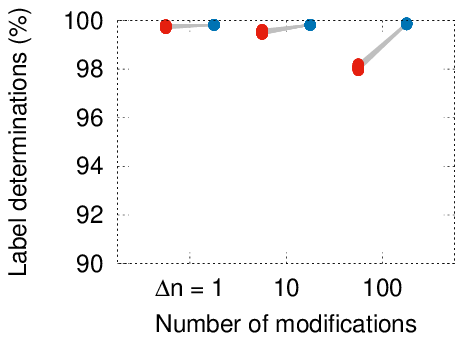}
\end{center}
\end{minipage}
&
\begin{minipage}{0.32\hsize}
\begin{center}
Spot modification

\includegraphics[clip,width=5cm]{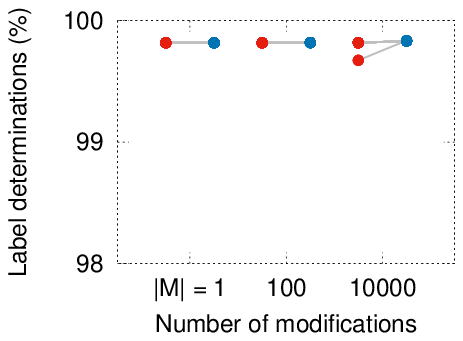}
\end{center}
\end{minipage}
\\ \hline
\multicolumn{3}{c}{Data set: {\tt news20}}
\\ \hline
\begin{minipage}{0.32\hsize}
\begin{center}
Feature modification

\includegraphics[clip,width=5cm]{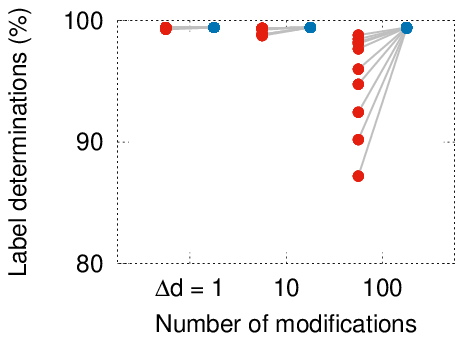}
\end{center}
\end{minipage}
&
\begin{minipage}{0.32\hsize}
\begin{center}
Instance modification

\includegraphics[clip,width=5cm]{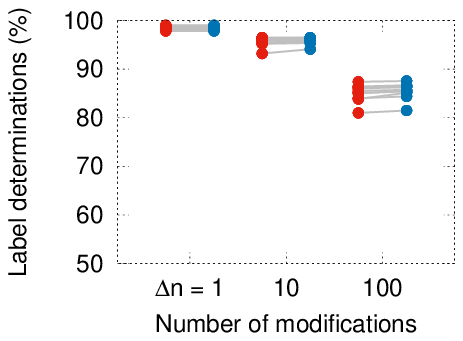}
\end{center}
\end{minipage}
&
\begin{minipage}{0.32\hsize}
\begin{center}
Spot modification

\includegraphics[clip,width=5cm]{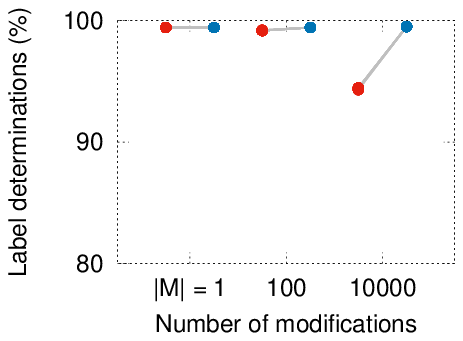}
\end{center}
\end{minipage}
\\ \hline
\multicolumn{3}{c}{Data set: {\tt real-sim}}
\\ \hline
\begin{minipage}{0.32\hsize}
\begin{center}
Feature modification

\includegraphics[clip,width=5cm]{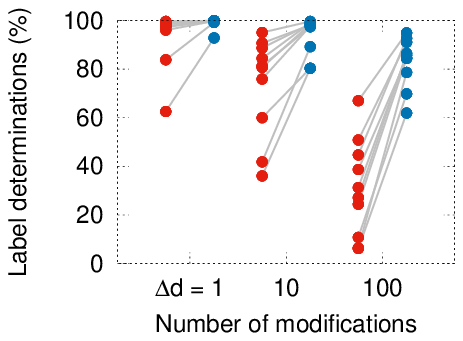}
\end{center}
\end{minipage}
&
\begin{minipage}{0.32\hsize}
\begin{center}
Instance modification

\includegraphics[clip,width=5cm]{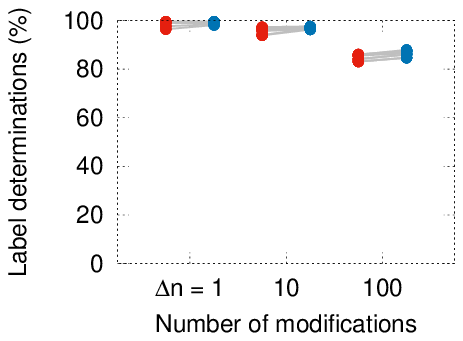}
\end{center}
\end{minipage}
&
\begin{minipage}{0.32\hsize}
\begin{center}
Spot modification

\includegraphics[clip,width=5cm]{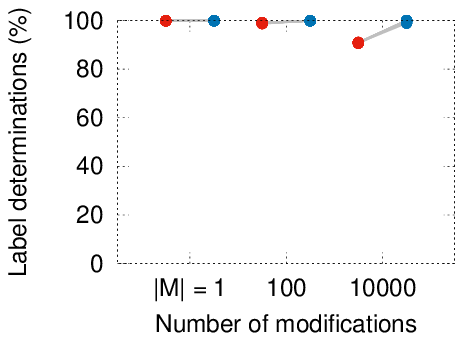}
\end{center}
\end{minipage}
\\ \hline
\multicolumn{3}{c}{Data set: {\tt rcv1-train}}
\\ \hline
\begin{minipage}{0.32\hsize}
\begin{center}
Feature modification

\includegraphics[clip,width=5cm]{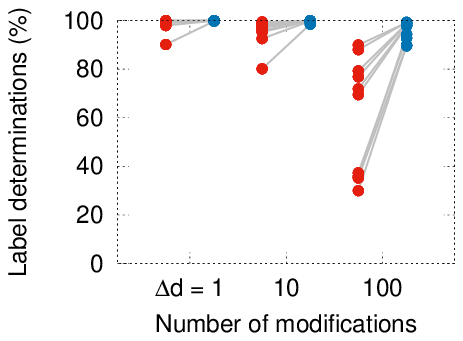}
\end{center}
\end{minipage}
&
\begin{minipage}{0.32\hsize}
\begin{center}
Instance modification

\includegraphics[clip,width=5cm]{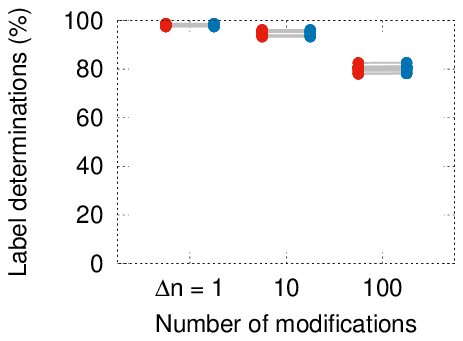}
\end{center}
\end{minipage}
&
\begin{minipage}{0.32\hsize}
\begin{center}
Spot modification

\includegraphics[clip,width=5cm]{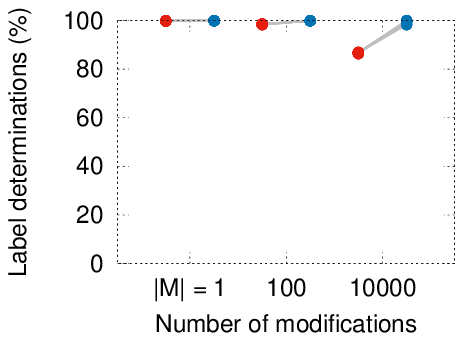}
\end{center}
\end{minipage}
\\ \hline
\end{tabular}
\end{figure}
\begin{figure}[p]
\caption{Label determination rate for $\lambda=1.0$}
\label{fig:experiment-all-4}
\begin{tabular}{ccc}
\hline
\multicolumn{3}{c}{Data set: {\tt kdda}}
\\ \hline
\begin{minipage}{0.32\hsize}
\begin{center}
Feature modification

\includegraphics[clip,width=5cm]{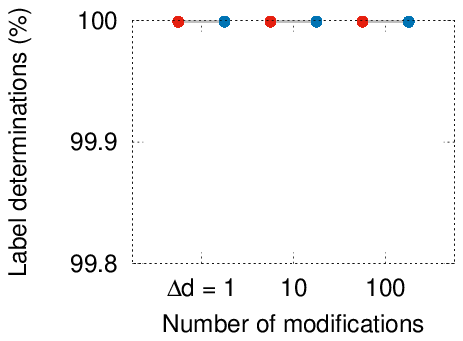}
\end{center}
\end{minipage}
&
\begin{minipage}{0.32\hsize}
\begin{center}
Instance modification

\includegraphics[clip,width=5cm]{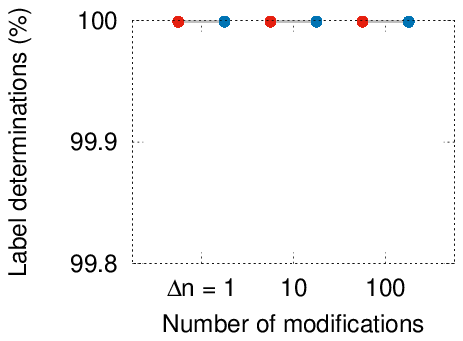}
\end{center}
\end{minipage}
&
\begin{minipage}{0.32\hsize}
\begin{center}
Spot modification

\includegraphics[clip,width=5cm]{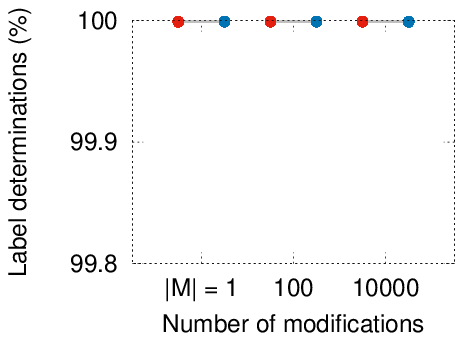}
\end{center}
\end{minipage}
\\ \hline
\multicolumn{3}{c}{Data set: {\tt url}}
\\ \hline
\begin{minipage}{0.32\hsize}
\begin{center}
Feature modification

\includegraphics[clip,width=5cm]{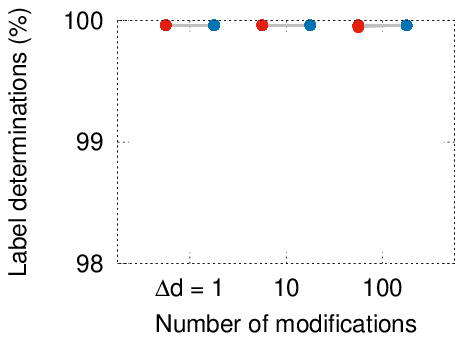}
\end{center}
\end{minipage}
&
\begin{minipage}{0.32\hsize}
\begin{center}
Instance modification

\includegraphics[clip,width=5cm]{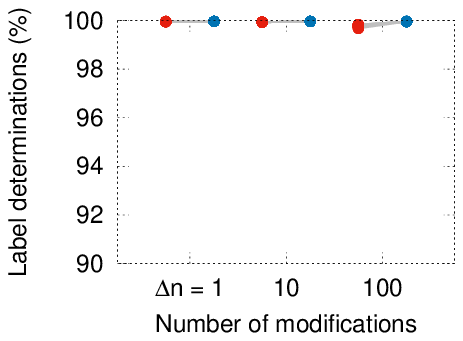}
\end{center}
\end{minipage}
&
\begin{minipage}{0.32\hsize}
\begin{center}
Spot modification

\includegraphics[clip,width=5cm]{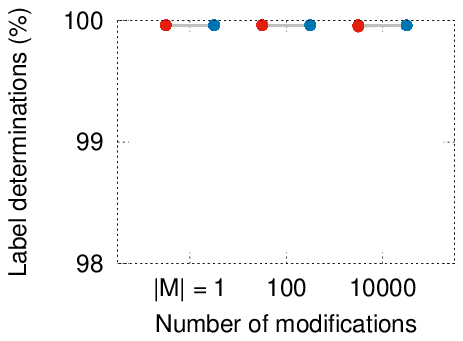}
\end{center}
\end{minipage}
\\ \hline
\multicolumn{3}{c}{Data set: {\tt news20}}
\\ \hline
\begin{minipage}{0.32\hsize}
\begin{center}
Feature modification

\includegraphics[clip,width=5cm]{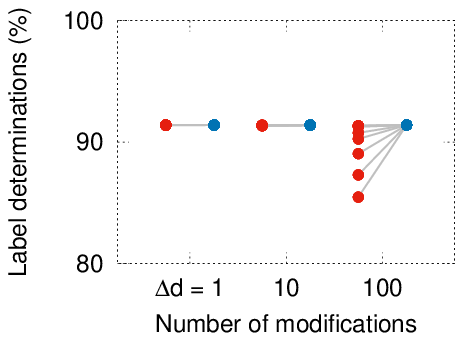}
\end{center}
\end{minipage}
&
\begin{minipage}{0.32\hsize}
\begin{center}
Instance modification

\includegraphics[clip,width=5cm]{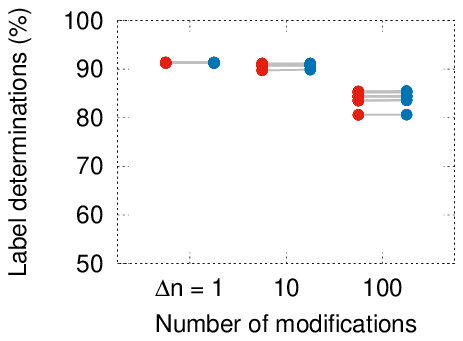}
\end{center}
\end{minipage}
&
\begin{minipage}{0.32\hsize}
\begin{center}
Spot modification

\includegraphics[clip,width=5cm]{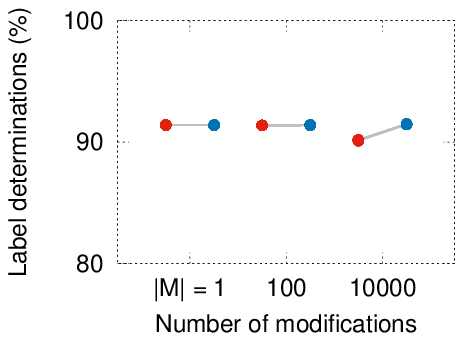}
\end{center}
\end{minipage}
\\ \hline
\multicolumn{3}{c}{Data set: {\tt real-sim}}
\\ \hline
\begin{minipage}{0.32\hsize}
\begin{center}
Feature modification

\includegraphics[clip,width=5cm]{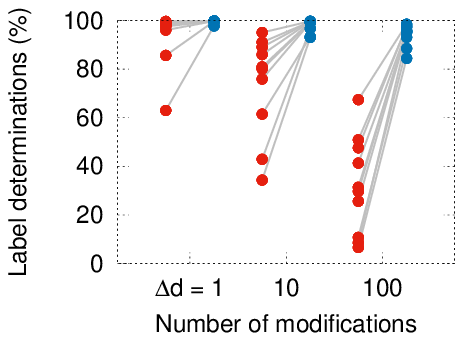}
\end{center}
\end{minipage}
&
\begin{minipage}{0.32\hsize}
\begin{center}
Instance modification

\includegraphics[clip,width=5cm]{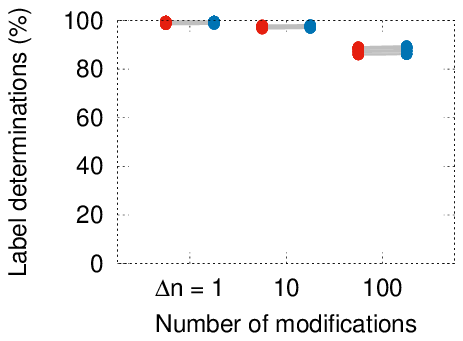}
\end{center}
\end{minipage}
&
\begin{minipage}{0.32\hsize}
\begin{center}
Spot modification

\includegraphics[clip,width=5cm]{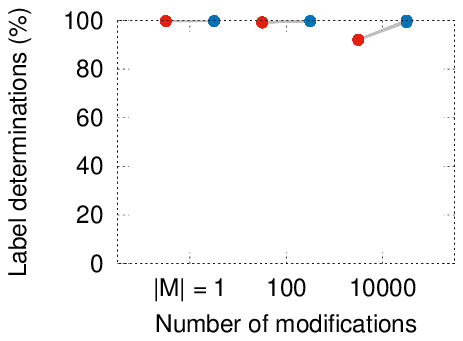}
\end{center}
\end{minipage}
\\ \hline
\multicolumn{3}{c}{Data set: {\tt rcv1-train}}
\\ \hline
\begin{minipage}{0.32\hsize}
\begin{center}
Feature modification

\includegraphics[clip,width=5cm]{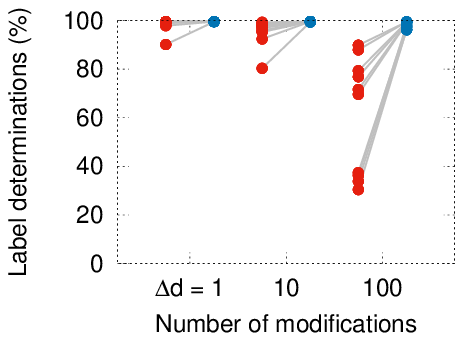}
\end{center}
\end{minipage}
&
\begin{minipage}{0.32\hsize}
\begin{center}
Instance modification

\includegraphics[clip,width=5cm]{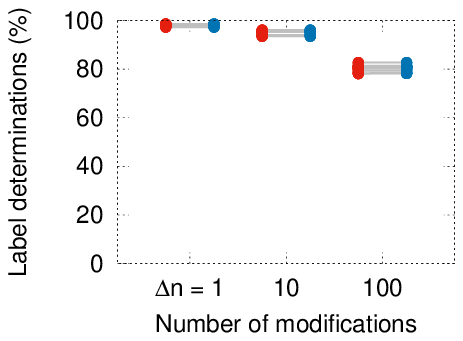}
\end{center}
\end{minipage}
&
\begin{minipage}{0.32\hsize}
\begin{center}
Spot modification

\includegraphics[clip,width=5cm]{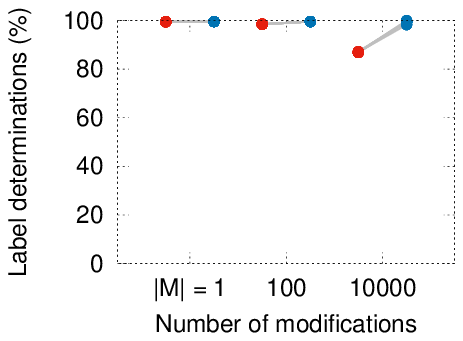}
\end{center}
\end{minipage}
\\ \hline
\end{tabular}
\end{figure}
\begin{figure}[p]
\caption{Upper bound of the parameter change $\|\tilde{\bm{w}}^* - \hat{\bm{w}}^*\|_2$ for $\lambda=0.001$}
\label{fig:experiment-all-5}
\begin{tabular}{ccc}
\hline
\multicolumn{3}{c}{Data set: {\tt kdda}}
\\ \hline
\begin{minipage}{0.32\hsize}
\begin{center}
Feature modification

\includegraphics[clip,width=5cm]{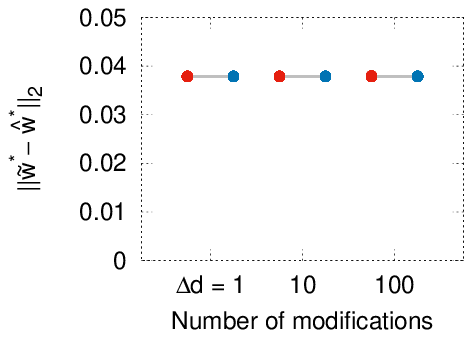}
\end{center}
\end{minipage}
&
\begin{minipage}{0.32\hsize}
\begin{center}
Instance modification

\includegraphics[clip,width=5cm]{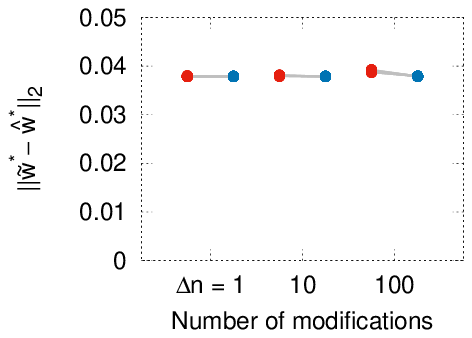}
\end{center}
\end{minipage}
&
\begin{minipage}{0.32\hsize}
\begin{center}
Spot modification

\includegraphics[clip,width=5cm]{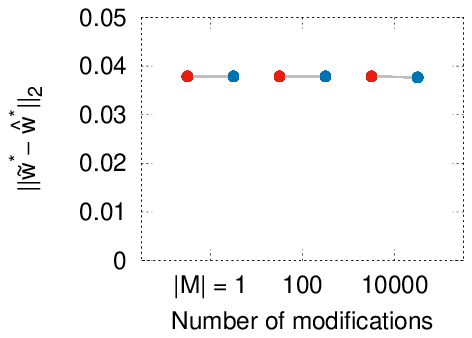}
\end{center}
\end{minipage}
\\ \hline
\multicolumn{3}{c}{Data set: {\tt url}}
\\ \hline
\begin{minipage}{0.32\hsize}
\begin{center}
Feature modification

\includegraphics[clip,width=5cm]{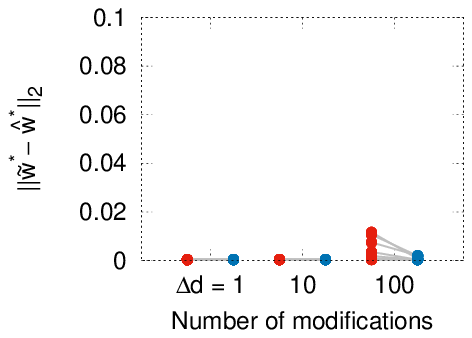}
\end{center}
\end{minipage}
&
\begin{minipage}{0.32\hsize}
\begin{center}
Instance modification

\includegraphics[clip,width=5cm]{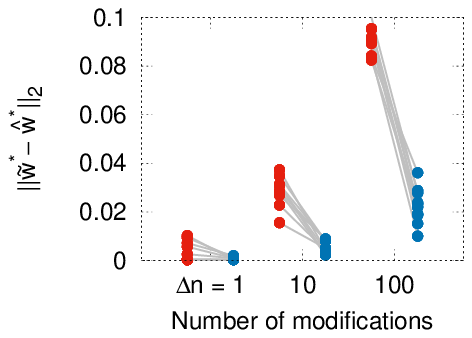}
\end{center}
\end{minipage}
&
\begin{minipage}{0.32\hsize}
\begin{center}
Spot modification

\includegraphics[clip,width=5cm]{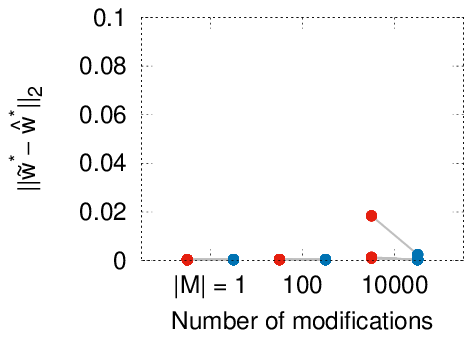}
\end{center}
\end{minipage}
\\ \hline
\multicolumn{3}{c}{Data set: {\tt news20}}
\\ \hline
\begin{minipage}{0.32\hsize}
\begin{center}
Feature modification

\includegraphics[clip,width=5cm]{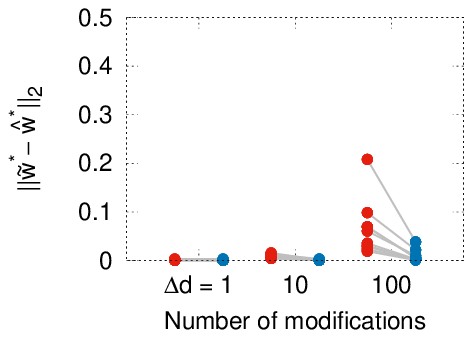}
\end{center}
\end{minipage}
&
\begin{minipage}{0.32\hsize}
\begin{center}
Instance modification

\includegraphics[clip,width=5cm]{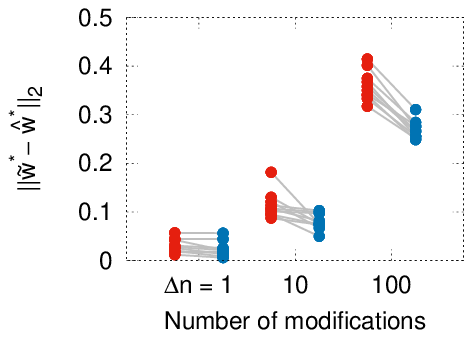}
\end{center}
\end{minipage}
&
\begin{minipage}{0.32\hsize}
\begin{center}
Spot modification

\includegraphics[clip,width=5cm]{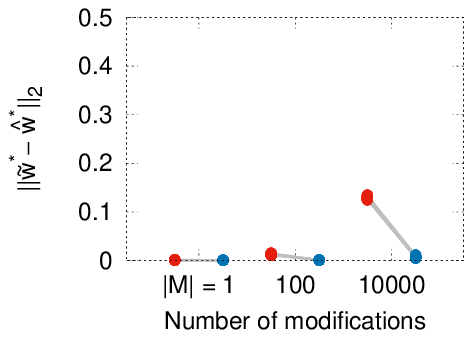}
\end{center}
\end{minipage}
\\ \hline
\multicolumn{3}{c}{Data set: {\tt real-sim}}
\\ \hline
\begin{minipage}{0.32\hsize}
\begin{center}
Feature modification

\includegraphics[clip,width=5cm]{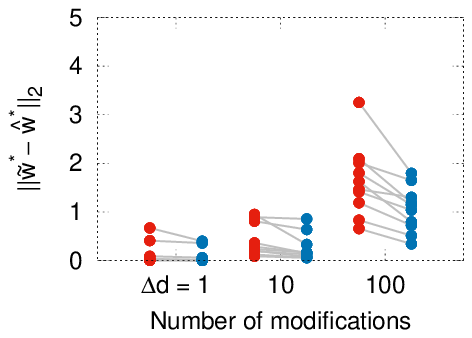}
\end{center}
\end{minipage}
&
\begin{minipage}{0.32\hsize}
\begin{center}
Instance modification

\includegraphics[clip,width=5cm]{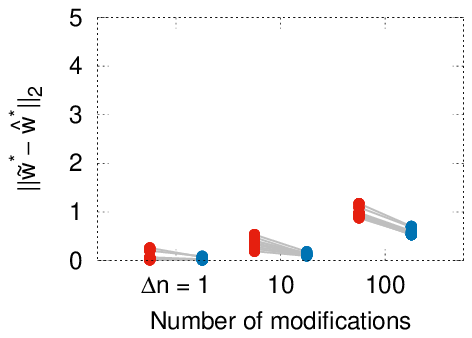}
\end{center}
\end{minipage}
&
\begin{minipage}{0.32\hsize}
\begin{center}
Spot modification

\includegraphics[clip,width=5cm]{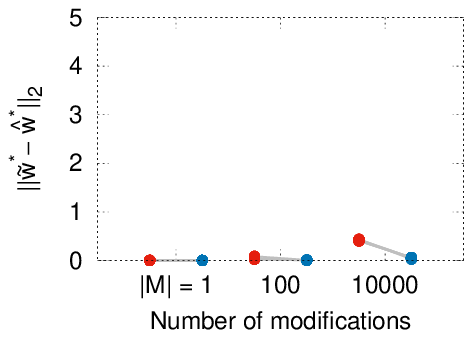}
\end{center}
\end{minipage}
\\ \hline
\multicolumn{3}{c}{Data set: {\tt rcv1-train}}
\\ \hline
\begin{minipage}{0.32\hsize}
\begin{center}
Feature modification

\includegraphics[clip,width=5cm]{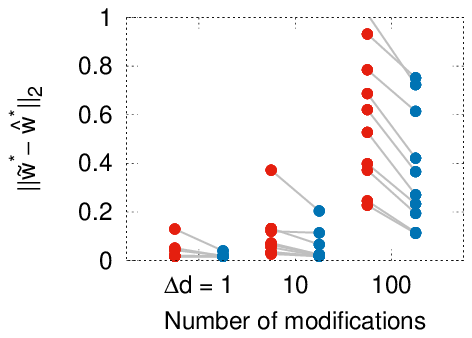}
\end{center}
\end{minipage}
&
\begin{minipage}{0.32\hsize}
\begin{center}
Instance modification

\includegraphics[clip,width=5cm]{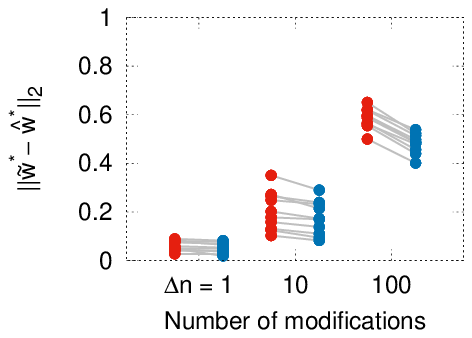}
\end{center}
\end{minipage}
&
\begin{minipage}{0.32\hsize}
\begin{center}
Spot modification

\includegraphics[clip,width=5cm]{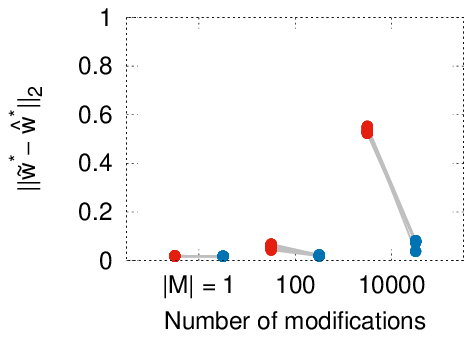}
\end{center}
\end{minipage}
\\ \hline
\end{tabular}
\end{figure}
\begin{figure}[p]
\caption{Upper bound of the parameter change $\|\tilde{\bm{w}}^* - \hat{\bm{w}}^*\|_2$ for $\lambda=0.01$}
\label{fig:experiment-all-6}
\begin{tabular}{ccc}
\hline
\multicolumn{3}{c}{Data set: {\tt kdda}}
\\ \hline
\begin{minipage}{0.32\hsize}
\begin{center}
Feature modification

\includegraphics[clip,width=5cm]{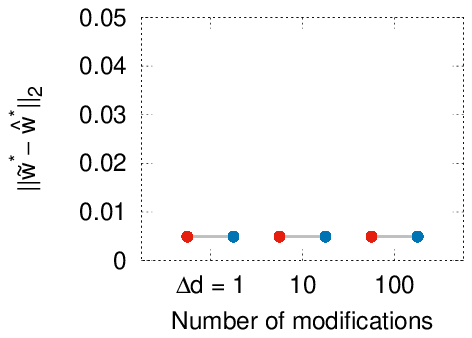}
\end{center}
\end{minipage}
&
\begin{minipage}{0.32\hsize}
\begin{center}
Instance modification

\includegraphics[clip,width=5cm]{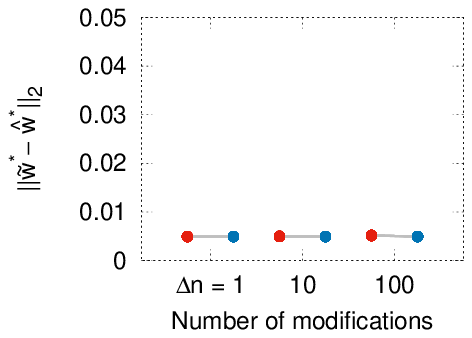}
\end{center}
\end{minipage}
&
\begin{minipage}{0.32\hsize}
\begin{center}
Spot modification

\includegraphics[clip,width=5cm]{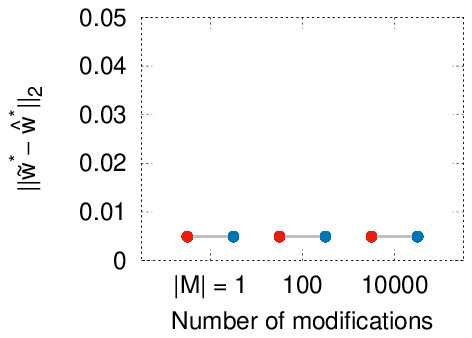}
\end{center}
\end{minipage}
\\ \hline
\multicolumn{3}{c}{Data set: {\tt url}}
\\ \hline
\begin{minipage}{0.32\hsize}
\begin{center}
Feature modification

\includegraphics[clip,width=5cm]{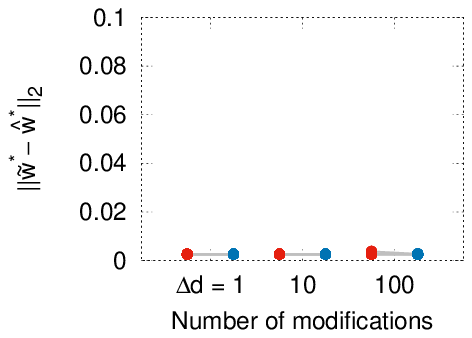}
\end{center}
\end{minipage}
&
\begin{minipage}{0.32\hsize}
\begin{center}
Instance modification

\includegraphics[clip,width=5cm]{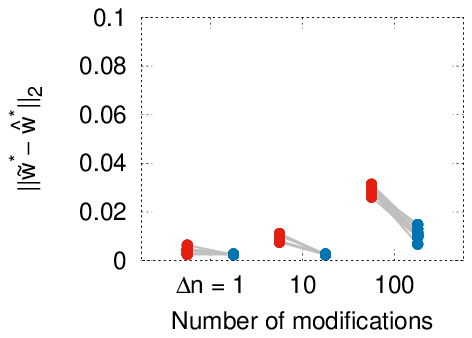}
\end{center}
\end{minipage}
&
\begin{minipage}{0.32\hsize}
\begin{center}
Spot modification

\includegraphics[clip,width=5cm]{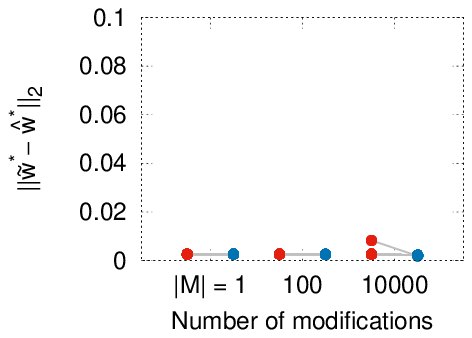}
\end{center}
\end{minipage}
\\ \hline
\multicolumn{3}{c}{Data set: {\tt news20}}
\\ \hline
\begin{minipage}{0.32\hsize}
\begin{center}
Feature modification

\includegraphics[clip,width=5cm]{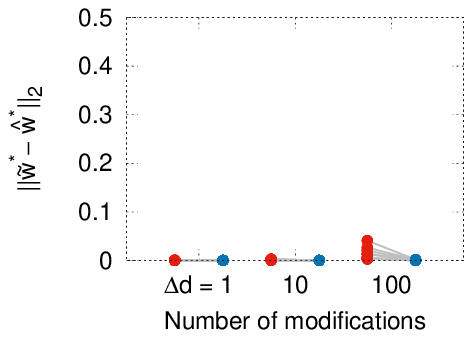}
\end{center}
\end{minipage}
&
\begin{minipage}{0.32\hsize}
\begin{center}
Instance modification

\includegraphics[clip,width=5cm]{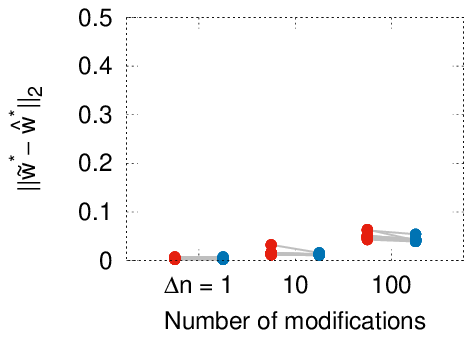}
\end{center}
\end{minipage}
&
\begin{minipage}{0.32\hsize}
\begin{center}
Spot modification

\includegraphics[clip,width=5cm]{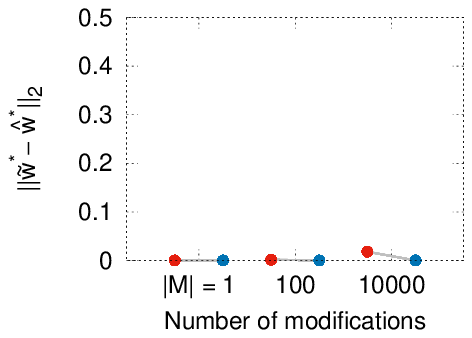}
\end{center}
\end{minipage}
\\ \hline
\multicolumn{3}{c}{Data set: {\tt real-sim}}
\\ \hline
\begin{minipage}{0.32\hsize}
\begin{center}
Feature modification

\includegraphics[clip,width=5cm]{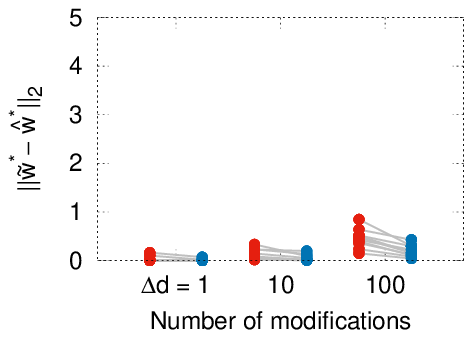}
\end{center}
\end{minipage}
&
\begin{minipage}{0.32\hsize}
\begin{center}
Instance modification

\includegraphics[clip,width=5cm]{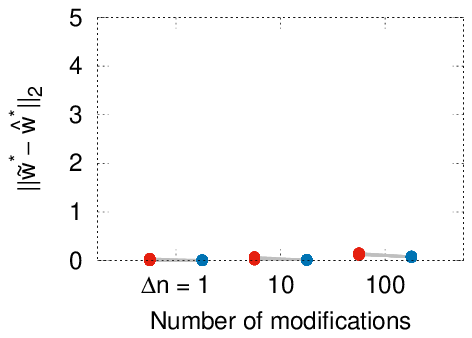}
\end{center}
\end{minipage}
&
\begin{minipage}{0.32\hsize}
\begin{center}
Spot modification

\includegraphics[clip,width=5cm]{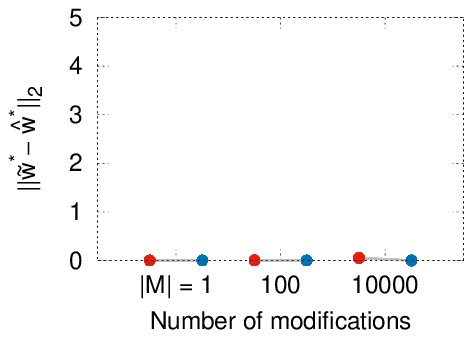}
\end{center}
\end{minipage}
\\ \hline
\multicolumn{3}{c}{Data set: {\tt rcv1-train}}
\\ \hline
\begin{minipage}{0.32\hsize}
\begin{center}
Feature modification

\includegraphics[clip,width=5cm]{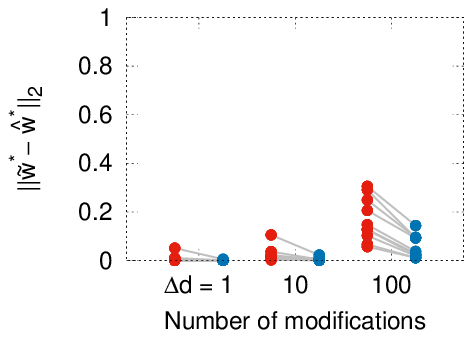}
\end{center}
\end{minipage}
&
\begin{minipage}{0.32\hsize}
\begin{center}
Instance modification

\includegraphics[clip,width=5cm]{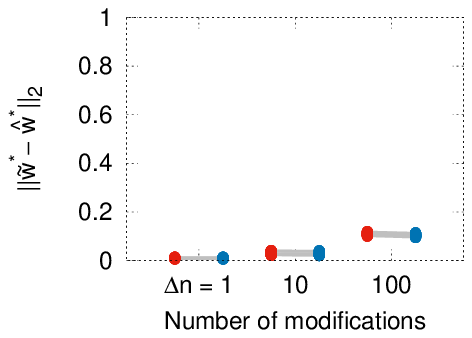}
\end{center}
\end{minipage}
&
\begin{minipage}{0.32\hsize}
\begin{center}
Spot modification

\includegraphics[clip,width=5cm]{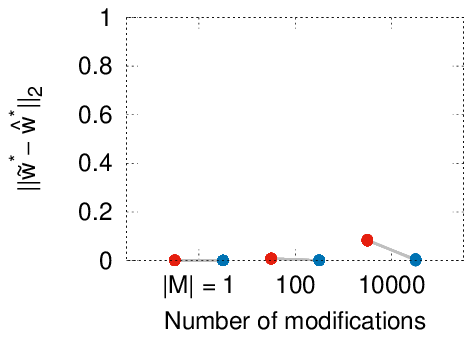}
\end{center}
\end{minipage}
\\ \hline
\end{tabular}
\end{figure}
\begin{figure}[p]
\caption{Upper bound of the parameter change $\|\tilde{\bm{w}}^* - \hat{\bm{w}}^*\|_2$ for $\lambda=0.1$}
\label{fig:experiment-all-7}
\begin{tabular}{ccc}
\hline
\multicolumn{3}{c}{Data set: {\tt kdda}}
\\ \hline
\begin{minipage}{0.32\hsize}
\begin{center}
Feature modification

\includegraphics[clip,width=5cm]{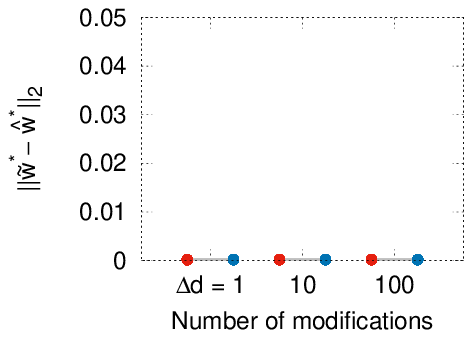}
\end{center}
\end{minipage}
&
\begin{minipage}{0.32\hsize}
\begin{center}
Instance modification

\includegraphics[clip,width=5cm]{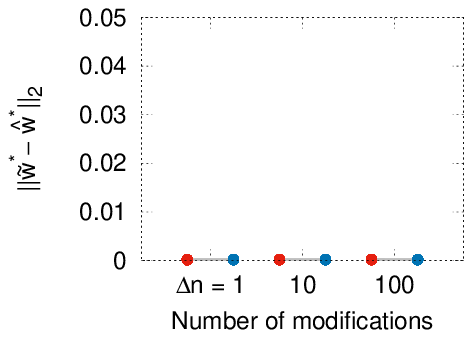}
\end{center}
\end{minipage}
&
\begin{minipage}{0.32\hsize}
\begin{center}
Spot modification

\includegraphics[clip,width=5cm]{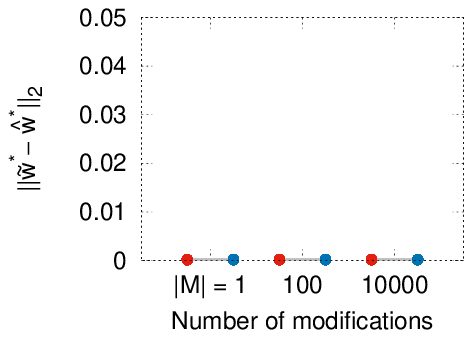}
\end{center}
\end{minipage}
\\ \hline
\multicolumn{3}{c}{Data set: {\tt url}}
\\ \hline
\begin{minipage}{0.32\hsize}
\begin{center}
Feature modification

\includegraphics[clip,width=5cm]{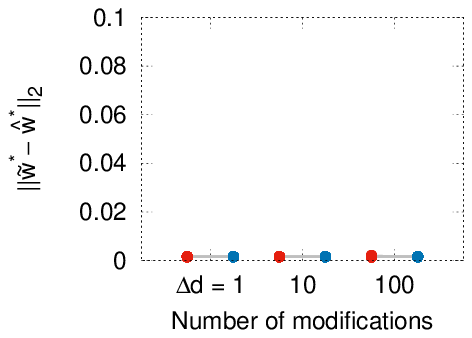}
\end{center}
\end{minipage}
&
\begin{minipage}{0.32\hsize}
\begin{center}
Instance modification

\includegraphics[clip,width=5cm]{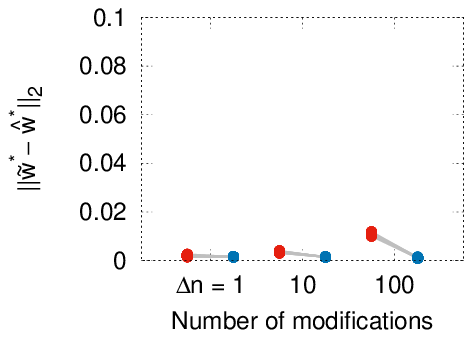}
\end{center}
\end{minipage}
&
\begin{minipage}{0.32\hsize}
\begin{center}
Spot modification

\includegraphics[clip,width=5cm]{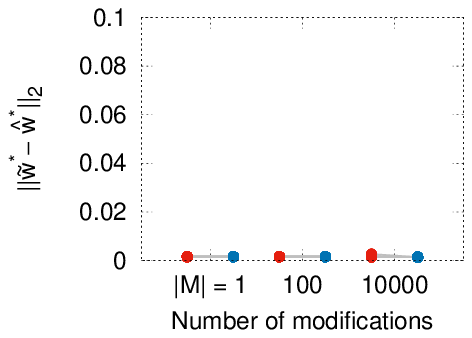}
\end{center}
\end{minipage}
\\ \hline
\multicolumn{3}{c}{Data set: {\tt news20}}
\\ \hline
\begin{minipage}{0.32\hsize}
\begin{center}
Feature modification

\includegraphics[clip,width=5cm]{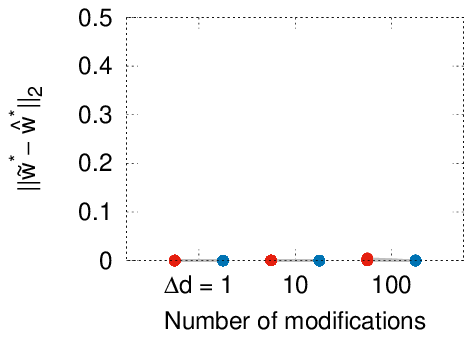}
\end{center}
\end{minipage}
&
\begin{minipage}{0.32\hsize}
\begin{center}
Instance modification

\includegraphics[clip,width=5cm]{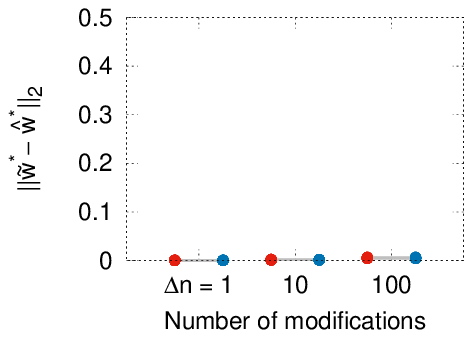}
\end{center}
\end{minipage}
&
\begin{minipage}{0.32\hsize}
\begin{center}
Spot modification

\includegraphics[clip,width=5cm]{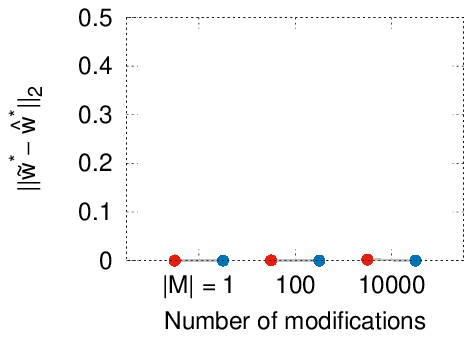}
\end{center}
\end{minipage}
\\ \hline
\multicolumn{3}{c}{Data set: {\tt real-sim}}
\\ \hline
\begin{minipage}{0.32\hsize}
\begin{center}
Feature modification

\includegraphics[clip,width=5cm]{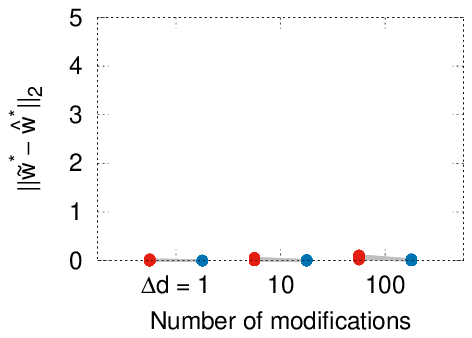}
\end{center}
\end{minipage}
&
\begin{minipage}{0.32\hsize}
\begin{center}
Instance modification

\includegraphics[clip,width=5cm]{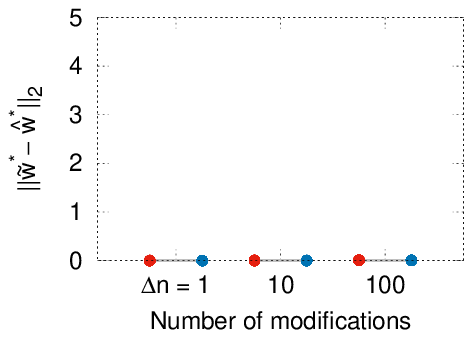}
\end{center}
\end{minipage}
&
\begin{minipage}{0.32\hsize}
\begin{center}
Spot modification

\includegraphics[clip,width=5cm]{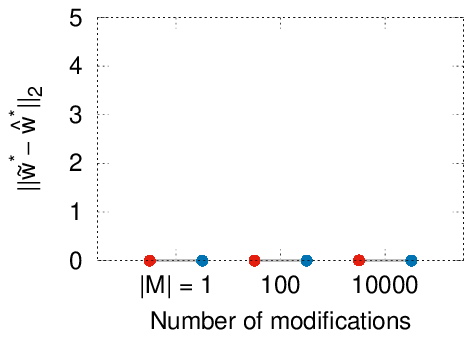}
\end{center}
\end{minipage}
\\ \hline
\multicolumn{3}{c}{Data set: {\tt rcv1-train}}
\\ \hline
\begin{minipage}{0.32\hsize}
\begin{center}
Feature modification

\includegraphics[clip,width=5cm]{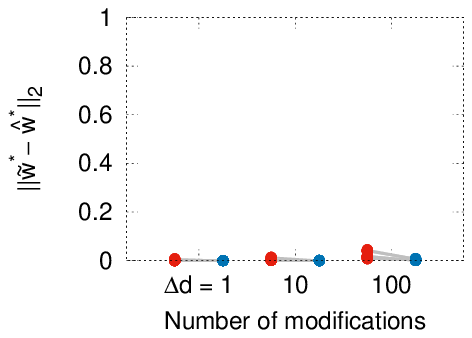}
\end{center}
\end{minipage}
&
\begin{minipage}{0.32\hsize}
\begin{center}
Instance modification

\includegraphics[clip,width=5cm]{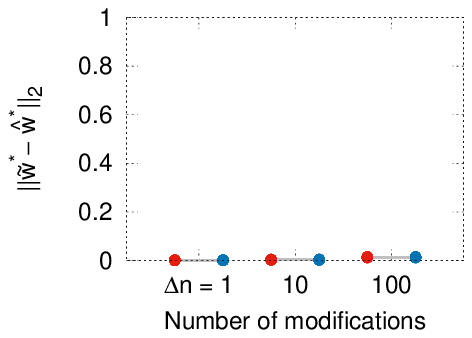}
\end{center}
\end{minipage}
&
\begin{minipage}{0.32\hsize}
\begin{center}
Spot modification

\includegraphics[clip,width=5cm]{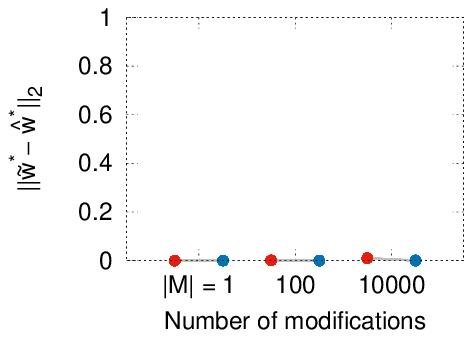}
\end{center}
\end{minipage}
\\ \hline
\end{tabular}
\end{figure}
\begin{figure}[p]
\caption{Upper bound of the parameter change $\|\tilde{\bm{w}}^* - \hat{\bm{w}}^*\|_2$ for $\lambda=1.0$}
\label{fig:experiment-all-8}
\begin{tabular}{ccc}
\hline
\multicolumn{3}{c}{Data set: {\tt kdda}}
\\ \hline
\begin{minipage}{0.32\hsize}
\begin{center}
Feature modification

\includegraphics[clip,width=5cm]{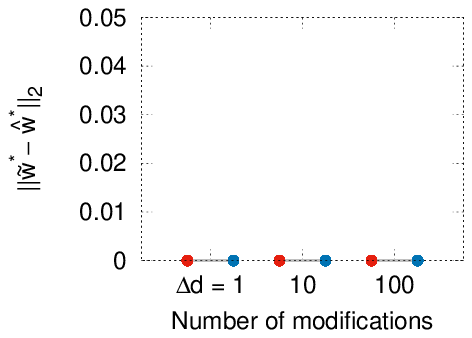}
\end{center}
\end{minipage}
&
\begin{minipage}{0.32\hsize}
\begin{center}
Instance modification

\includegraphics[clip,width=5cm]{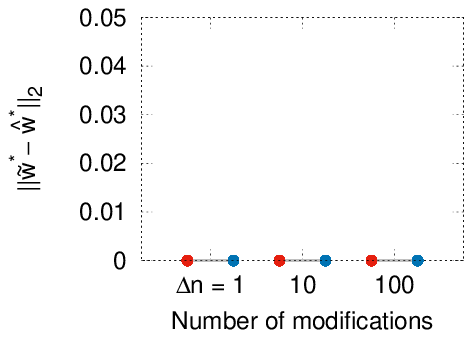}
\end{center}
\end{minipage}
&
\begin{minipage}{0.32\hsize}
\begin{center}
Spot modification

\includegraphics[clip,width=5cm]{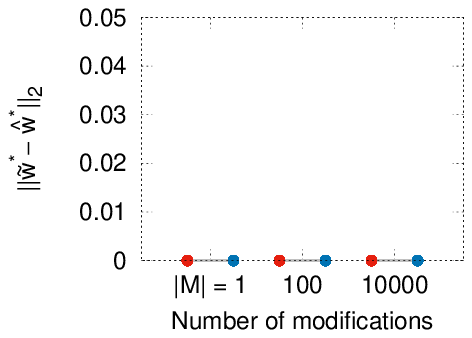}
\end{center}
\end{minipage}
\\ \hline
\multicolumn{3}{c}{Data set: {\tt url}}
\\ \hline
\begin{minipage}{0.32\hsize}
\begin{center}
Feature modification

\includegraphics[clip,width=5cm]{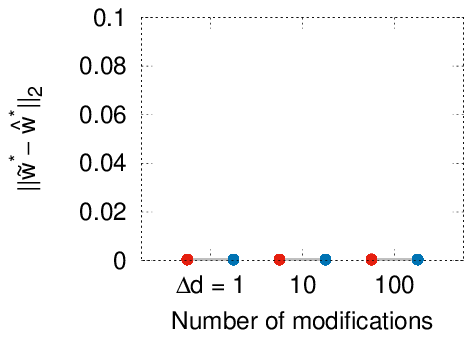}
\end{center}
\end{minipage}
&
\begin{minipage}{0.32\hsize}
\begin{center}
Instance modification

\includegraphics[clip,width=5cm]{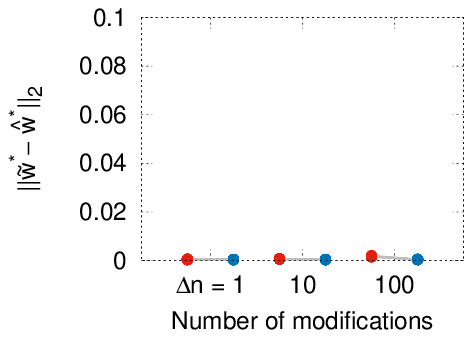}
\end{center}
\end{minipage}
&
\begin{minipage}{0.32\hsize}
\begin{center}
Spot modification

\includegraphics[clip,width=5cm]{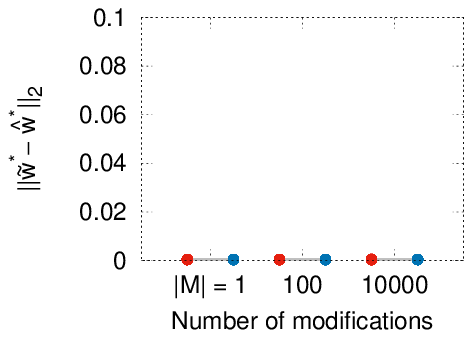}
\end{center}
\end{minipage}
\\ \hline
\multicolumn{3}{c}{Data set: {\tt news20}}
\\ \hline
\begin{minipage}{0.32\hsize}
\begin{center}
Feature modification

\includegraphics[clip,width=5cm]{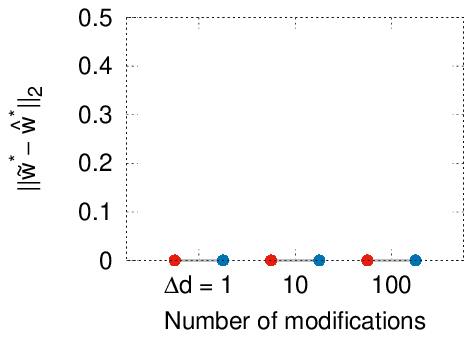}
\end{center}
\end{minipage}
&
\begin{minipage}{0.32\hsize}
\begin{center}
Instance modification

\includegraphics[clip,width=5cm]{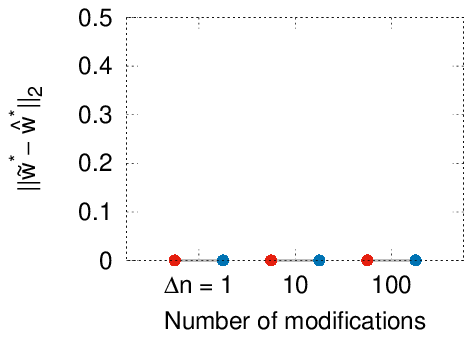}
\end{center}
\end{minipage}
&
\begin{minipage}{0.32\hsize}
\begin{center}
Spot modification

\includegraphics[clip,width=5cm]{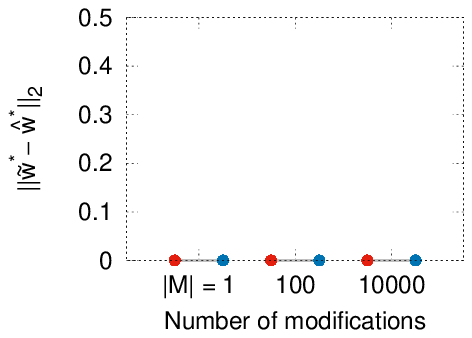}
\end{center}
\end{minipage}
\\ \hline
\multicolumn{3}{c}{Data set: {\tt real-sim}}
\\ \hline
\begin{minipage}{0.32\hsize}
\begin{center}
Feature modification

\includegraphics[clip,width=5cm]{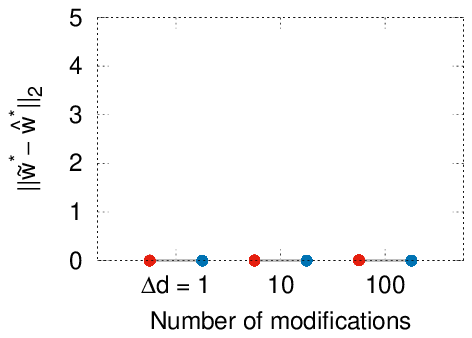}
\end{center}
\end{minipage}
&
\begin{minipage}{0.32\hsize}
\begin{center}
Instance modification

\includegraphics[clip,width=5cm]{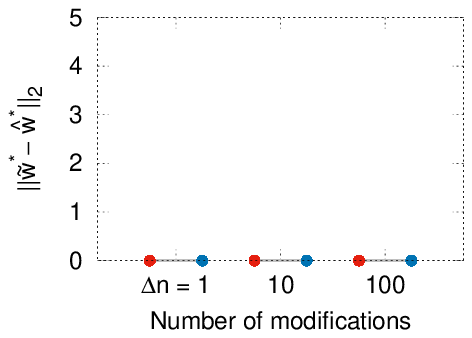}
\end{center}
\end{minipage}
&
\begin{minipage}{0.32\hsize}
\begin{center}
Spot modification

\includegraphics[clip,width=5cm]{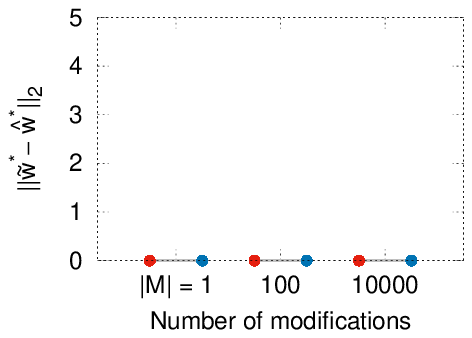}
\end{center}
\end{minipage}
\\ \hline
\multicolumn{3}{c}{Data set: {\tt rcv1-train}}
\\ \hline
\begin{minipage}{0.32\hsize}
\begin{center}
Feature modification

\includegraphics[clip,width=5cm]{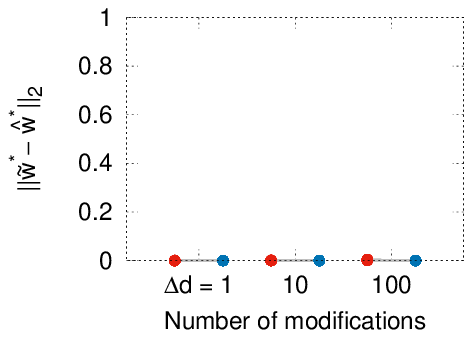}
\end{center}
\end{minipage}
&
\begin{minipage}{0.32\hsize}
\begin{center}
Instance modification

\includegraphics[clip,width=5cm]{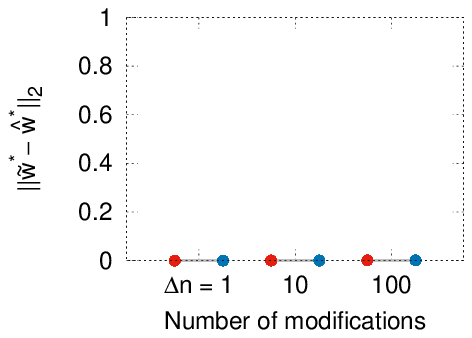}
\end{center}
\end{minipage}
&
\begin{minipage}{0.32\hsize}
\begin{center}
Spot modification

\includegraphics[clip,width=5cm]{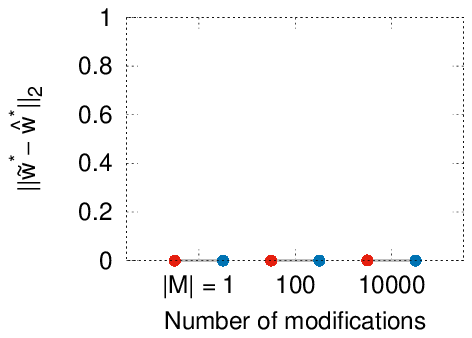}
\end{center}
\end{minipage}
\\ \hline
\end{tabular}
\end{figure}

\begin{table}[tp]
\caption{The ratio of the computation time of the bound computation to re-training (with warm-start). In all settings, the former cost is almost negligible.}
\label{table:computation-time-all}
\begin{center}
{\scriptsize
\begin{tabular}{c|rrr|rrr|rrr}
\hline
& \multicolumn{3}{c|}{(a) spot modification} & \multicolumn{3}{c|}{(b) instance modification} & \multicolumn{3}{c}{(c) feature modification} \\
 & \multicolumn{1}{c}{1} & \multicolumn{1}{c}{100} & \multicolumn{1}{c|}{10000}
 & \multicolumn{1}{c}{1} & \multicolumn{1}{c}{10} & \multicolumn{1}{c|}{100}
 & \multicolumn{1}{c}{1} & \multicolumn{1}{c}{10} & \multicolumn{1}{c}{100} \\
\hline
\multicolumn{10}{c}{$\lambda = 01$} \\
\hline
{\tt kdd-a} & $3\times 10^{-6}$ & $1\times 10^{-4}$ & $1\times 10^{-4}$ & $5\times 10^{-7}$ & $2\times 10^{-4}$ & $1\times 10^{-2}$ & $2\times 10^{-7}$ & $5\times 10^{-7}$ & $5\times 10^{-5}$ \\
{\tt url} & $5\times 10^{-8}$ & $2\times 10^{-5}$ & $5\times 10^{-5}$ & $3\times 10^{-7}$ & $2\times 10^{-3}$ & $3\times 10^{-2}$ & $3\times 10^{-5}$ & $3\times 10^{-6}$ & $4\times 10^{-5}$ \\
{\tt news20} & $2\times 10^{-6}$ & $3\times 10^{-5}$ & $5\times 10^{-3}$ & $7\times 10^{-5}$ & $2\times 10^{-3}$ & $6\times 10^{-2}$ & $4\times 10^{-6}$ & $1\times 10^{-5}$ & $4\times 10^{-4}$ \\
{\tt real-sim} & $9\times 10^{-6}$ & $1\times 10^{-4}$ & $7\times 10^{-3}$ & $5\times 10^{-5}$ & $4\times 10^{-4}$ & $2\times 10^{-3}$ & $1\times 10^{-4}$ & $6\times 10^{-4}$ & $2\times 10^{-3}$ \\
{\tt rcv1-train} & $2\times 10^{-5}$ & $2\times 10^{-4}$ & $1\times 10^{-2}$ & $2\times 10^{-4}$ & $5\times 10^{-4}$ & $6\times 10^{-3}$ & $5\times 10^{-5}$ & $2\times 10^{-4}$ & $2\times 10^{-3}$ \\
\hline
\multicolumn{10}{c}{$\lambda = 0$} \\
\hline
{\tt kdd-a} & $6\times 10^{-6}$ & $3\times 10^{-6}$ & $2\times 10^{-4}$ & $1\times 10^{-6}$ & $1\times 10^{-4}$ & $1\times 10^{-2}$ & $5\times 10^{-7}$ & $1\times 10^{-6}$ & $2\times 10^{-4}$ \\
{\tt url} & $1\times 10^{-7}$ & $3\times 10^{-6}$ & $8\times 10^{-5}$ & $6\times 10^{-6}$ & $3\times 10^{-3}$ & $5\times 10^{-2}$ & $9\times 10^{-8}$ & $1\times 10^{-6}$ & $3\times 10^{-5}$ \\
{\tt news20} & $3\times 10^{-6}$ & $6\times 10^{-5}$ & $7\times 10^{-3}$ & $8\times 10^{-4}$ & $3\times 10^{-3}$ & $5\times 10^{-2}$ & $1\times 10^{-5}$ & $4\times 10^{-5}$ & $7\times 10^{-4}$ \\
{\tt real-sim} & $1\times 10^{-5}$ & $2\times 10^{-4}$ & $9\times 10^{-3}$ & $7\times 10^{-5}$ & $4\times 10^{-4}$ & $3\times 10^{-3}$ & $2\times 10^{-4}$ & $1\times 10^{-3}$ & $4\times 10^{-3}$ \\
{\tt rcv1-train} & $2\times 10^{-5}$ & $3\times 10^{-4}$ & $2\times 10^{-2}$ & $2\times 10^{-4}$ & $1\times 10^{-3}$ & $9\times 10^{-3}$ & $8\times 10^{-5}$ & $4\times 10^{-4}$ & $2\times 10^{-3}$ \\
\hline
\multicolumn{10}{c}{$\lambda = 0.1$} \\
\hline
{\tt kdd-a} & $2\times 10^{-7}$ & $4\times 10^{-6}$ & $1\times 10^{-4}$ & $4\times 10^{-7}$ & $5\times 10^{-6}$ & $2\times 10^{-2}$ & $2\times 10^{-7}$ & $2\times 10^{-6}$ & $1\times 10^{-5}$ \\
{\tt url} & $1\times 10^{-7}$ & $2\times 10^{-6}$ & $8\times 10^{-5}$ & $4\times 10^{-7}$ & $4\times 10^{-4}$ & $5\times 10^{-2}$ & $2\times 10^{-7}$ & $1\times 10^{-6}$ & $5\times 10^{-5}$ \\
{\tt news20} & $4\times 10^{-6}$ & $4\times 10^{-5}$ & $8\times 10^{-3}$ & $6\times 10^{-5}$ & $3\times 10^{-3}$ & $1\times 10^{-2}$ & $6\times 10^{-6}$ & $3\times 10^{-5}$ & $7\times 10^{-4}$ \\
{\tt real-sim} & $1\times 10^{-5}$ & $2\times 10^{-4}$ & $2\times 10^{-2}$ & $8\times 10^{-5}$ & $5\times 10^{-4}$ & $5\times 10^{-3}$ & $2\times 10^{-4}$ & $1\times 10^{-3}$ & $6\times 10^{-3}$ \\
{\tt rcv1-train} & $3\times 10^{-5}$ & $4\times 10^{-4}$ & $2\times 10^{-2}$ & $3\times 10^{-4}$ & $2\times 10^{-3}$ & $9\times 10^{-3}$ & $9\times 10^{-5}$ & $5\times 10^{-4}$ & $1\times 10^{-3}$ \\
\hline
\multicolumn{10}{c}{$\lambda = 1$} \\
\hline
{\tt kdd-a} & $3\times 10^{-7}$ & $3\times 10^{-6}$ & $2\times 10^{-4}$ & $5\times 10^{-7}$ & $6\times 10^{-6}$ & $2\times 10^{-2}$ & $2\times 10^{-7}$ & $2\times 10^{-6}$ & $2\times 10^{-5}$ \\
{\tt url} & $1\times 10^{-7}$ & $2\times 10^{-6}$ & $3\times 10^{-4}$ & $5\times 10^{-7}$ & $8\times 10^{-3}$ & $5\times 10^{-2}$ & $2\times 10^{-7}$ & $2\times 10^{-6}$ & $7\times 10^{-5}$ \\
{\tt news20} & $4\times 10^{-6}$ & $1\times 10^{-4}$ & $6\times 10^{-3}$ & $2\times 10^{-4}$ & $2\times 10^{-3}$ & $9\times 10^{-3}$ & $6\times 10^{-6}$ & $2\times 10^{-5}$ & $1\times 10^{-3}$ \\
{\tt real-sim} & $3\times 10^{-5}$ & $2\times 10^{-4}$ & $1\times 10^{-2}$ & $1\times 10^{-4}$ & $5\times 10^{-4}$ & $5\times 10^{-3}$ & $2\times 10^{-4}$ & $2\times 10^{-3}$ & $8\times 10^{-3}$ \\
{\tt rcv1-train} & $5\times 10^{-5}$ & $4\times 10^{-4}$ & $2\times 10^{-2}$ & $3\times 10^{-4}$ & $2\times 10^{-3}$ & $2\times 10^{-2}$ & $2\times 10^{-4}$ & $4\times 10^{-4}$ & $4\times 10^{-3}$ \\
\hline
\end{tabular}
}
\end{center}
\end{table}

%%%%%%%%%%%%%%%%%%%%%%%%%
%%%%%%%%%%%%%%%%%%%%%%%%%
\section{Conclusions} \label{sec:conc}
%%%%%%%%%%%%%%%%%%%%%%%%%
%%%%%%%%%%%%%%%%%%%%%%%%%
In this paper, we present a method for quickly incorporating the data modification effect into the classifier. The proposed method provides bounds on the optimal solution with the cost proportional to the size of the data modification. 
%
%Numerical experiments indicate that the method is highly useful when the number of modified elements are much smaller than the entire dataset size.
%
The experimental results indicate that our bound computation method in \S\ref{sec:main} is highly effective when the number of modified elements is much smaller than the entire dataset size. 
%
%It provides sufficiently tight bounds while the computational cost is negligible comparable with the cost of re-training the classifier.
%
In addition, partial optimization approach is also effective especially when the bounds in Theorem~\ref{theo:main} is not good enough.

\clearpage
\bibliography{paper}
\bibliographystyle{unsrt}
\clearpage
\appendix
\section*{Appendix}
\section{Proofs} \label{app:proof}

Before showing the proofs, we state some lemmas used in the proofs.

\begin{lemma} (Theorem 2 in \cite{ndiaye2015gap})
\label{lemma:dual_gap_bound}
Let $\hat{\bm w}$ be an any solution of the primal problem $P$, and $\hat{\bm \alpha} $ be an any feasible solution of the dual problem $D$.
If $D$ is $\gamma$-strongly concave then the optimal solution of dual problem $\bm \alpha^*$
is within in the sphere
$\Theta_{\bm \alpha^*} := \{\bm \alpha ~|~ \| \bm \alpha - \hat{\bm \alpha} \|_2 \le \sqrt{2 \gamma^{-1} (P(\hat{\bm w}) - D(\hat{\bm \alpha}) )} \} $.
\end{lemma}

\begin{lemma} (Corollary 4.2 in \cite{Zimmert2015})
\label{lemma:primal_gap_bound}
Let $\hat{\bm w}$ be an any solution of the primal problem $P$, and $\hat{\bm \alpha} $ be an any feasible solution of the dual problem $D$.
If $P$ is $\lambda$-strongly convex then the optimal solution of primal problem ${\bm w}^*$
is within in the sphere $\Theta_{\bm w^*} := \{ \bm w ~|~ \| \bm w - \hat{\bm w} \|_2 \le \sqrt{2 \lambda^{-1} (P(\hat{\bm w}) - D(\hat{\bm \alpha}) ) } \} $.
\end{lemma}

\begin{lemma}
\label{lemma:score_bound_by_shpere}
Let $\bm \eta, \bm c \in \RR^d $ be arbitrary vectors and a scalar $p > 0$.
Then, the following optimization problems have closed form solutions as follows:
\begin{align*}
  \bm \eta^\top \bm c - \sqrt{p} \| \bm \eta \|_2  = \min_{\bm q \in \RR^d} \bm \eta^\top \bm q
  {\rm ~~s.t.~~ } \| \bm q - \bm c \|_2 \le \sqrt{p}, \\
  \bm \eta^\top \bm c + \sqrt{p} \| \bm \eta \|_2  = \max_{\bm q \in \RR^d} \bm \eta^\top \bm q
  {\rm ~~s.t.~~ } \| \bm q - \bm c \|_2 \le \sqrt{p}
\end{align*}
\end{lemma}
\begin{proof}
Using a Lagrange multiplier method, we can easily prove this Lemma.
\end{proof}

\subsection{Proof of Theorem 1}

\begin{proof}
From the assumptions with regard to the penalty function $\psi$,
\begin{align*}
  % \label{eq:decomposablity_psi_conj}
  \psi^*(\bm w) &= \sup_{\bm v \in \RR^d}  \left\{ \bm w^\top \bm v - \psi(\bm v) \right\}  \\
  &= \sup_{\bm v \in \RR^d} \left\{\sum_{j \in [d] } w_j v_j - \sum_{j \in [d]} \psi_j(v_j) \right\}  \\
  &= \sum_{j \in [d]} \left\{ \sup_{v_j \in \RR} \{ w_j v_j - \psi_j(v_j) \} \right\}.
\end{align*}
Let $\psi^*_j(w_j) $ be a function such that $\sup_{v_j \in \RR} \{ w_j v_j - \psi_j(v_j) \}$,
$\psi^*(\bm w) = \sum_{j \in [d]} \psi^*_j(w_j) $.
Thus, from KKT condition,
\begin{align}
  \label{eq:kkt_wj}
  \tilde{w}^*_j \in \partial \psi^*_j (n^{-1} \tilde{\bm z}_{ \cdot j}^\top \tilde{\bm \alpha}^* ).
\end{align}
The old and the new primal and dual problems are respectively denoted as
$\hat{P}, \tilde{P}, \hat{D}, \tilde{D} $.
Since $\hat{\bm w}^*$ and $ \hat{\bm \alpha}^* $ is optimal solution of $\hat{P}$ and $\hat{D}$, respectively,
% and from {\bf Property C},
\begin{align}
  \label{eq:defi_diff_gap}
  \tilde{P}(\hat{\bm w}^*) - \tilde{D}(\hat{\bm \alpha}^*)
  &= n^{-1} \sum_{i \in [n] } \phi ( \hat{\bm z}_{i \cdot}^\top \hat{\bm w}^*)
  + n^{-1} \sum_{i \in \cM_i } ( \phi ( \tilde{\bm z}_{i \cdot}^\top \hat{\bm w}^* ) - \phi (\hat{\bm z}_{i \cdot}^\top \hat{\bm w}^*) )+ \psi(\hat{\bm w}^*) \nonumber \\
  &+ n^{-1} \sum_{i \in [n]} \phi^*(-\alpha^*_i)
  + \sum_{ j \in [d] } ( \psi^*_j ( n^{-1} \hat{\bm z}_{ \cdot j}^\top \hat{\bm \alpha}^* )  )
  + \sum_{ j \in \cM_j } ( \psi^*_j ( n^{-1} \tilde{\bm z}_{ \cdot j}^\top \hat{\bm \alpha}^* ) - \psi^*_j ( n^{-1} \hat{\bm z}_{ \cdot j}^\top \hat{\bm \alpha}^* ) ) \nonumber \\
  &= n^{-1} \sum_{i \in \cM_i } ( \phi ( \tilde{\bm z}_{i \cdot}^\top \hat{\bm w}^* ) - \phi (\hat{\bm z}_{i \cdot}^\top \hat{\bm w}^*) )
  + \sum_{ j \in \cM_j } ( \psi^*_j ( n^{-1} \tilde{\bm z}_{ \cdot j}^\top \hat{\bm \alpha}^* ) - \psi^*_j ( n^{-1} \hat{\bm z}_{ \cdot j}^\top \hat{\bm \alpha}^* ) ).
\end{align}
For a subset $\cJ \subseteq [d] $,
we define
$\cM(\cdot, \cJ) := \{ i \in [n] ~|~ \exists j \in \cJ {\rm ~s.t.~} (i, j) \in \cM \} $.
Similarly, for a subset $\cI \subseteq [n] $, we define
$\cM(\cI, \cdot) := \{ j \in [n] ~|~ \exists i \in \cI {\rm ~s.t.~} (i, j) \in \cM \} $.
Since $\{\hat{\bm z}_{i \cdot}^\top \hat{\bm w}^*\}_{i \in [n] } $
are stored in the memory and
$\tilde{\bm z}_{i \cdot} ^\top \hat{\bm w}^* = \hat{\bm z}_{i \cdot}^\top \hat{\bm w}^* + \sum_{j \in \cM(i, \cdot)} \hat{w}^*_j (\tilde{z}_{ij} - \hat{z}_{ij} ) $,
$\{\tilde{\bm z}_{i \cdot} ^\top \hat{\bm w}^*\}_{i \in \cM_i} $ is computed with
$\cO(\sum_{i \in \cM_i} | \cM(i, \cdot) | ) = \cO(|\cM|) $ time complexity.
Similarly, $\{\tilde{\bm z}_{\cdot j} ^\top \hat{\bm w}^*\}_{j \in \cM_j} $ is computed with
$\cO(\sum_{j \in \cM_j} | \cM(\cdot, j) | ) = \cO(|\cM|) $ time complexity because
$\tilde{\bm z}_{\cdot j} ^\top \hat{\bm w}^* = \hat{\bm z}_{j \cdot}^\top \hat{\bm \alpha}^* + \sum_{i \in \cM(\cdot, j)} \hat{\alpha}^*_i (\tilde{z}_{ij} - \hat{z}_{ij} )$.
Also, the computational cost of
$\tilde{G}(\hat{\bm w}^*, \hat{\bm \alpha}^*)$ is
$\cO(|\cM|)$.

$(i)$ The case where the loss function has {\bf Property A}:

From {\bf Property A}, $\phi^*$ is $\gamma$-strongly convex and so $\tilde{D}$ is $\gamma/n$-strongly concave
and therefore,
$ \tilde{\bm \alpha}^* \in \Theta_{\tilde{\bm \alpha}^*}
:= \{ \bm \alpha ~|~ \| \bm \alpha - \hat{\bm \alpha}^* \|_2 \le \sqrt{2 \gamma^{-1} (\tilde{P}(\hat{\bm w}^*) - \tilde{D}(\hat{\bm \alpha}^*) )} \} $
by using Lemma \ref{lemma:dual_gap_bound}.
Since $\Theta_{\tilde{\bm \alpha}^*}$ is a sphere,
the lower and upper bounds of the inner product of arbitrary vector
$\bm \eta \in \RR^n$ and
$\tilde{\bm \alpha}^*$ are given in closed form as follows by using Lemma \ref{lemma:score_bound_by_shpere},
\begin{align}
  \label{eq:Ld_score}
  \bm \eta^\top \tilde{\bm \alpha}^* \ge \bm \eta^\top \hat{\bm \alpha}^* - \|\bm \eta\|_2 \sqrt{2n\gamma^{-1}(\tilde{P}(\hat{\bm w}^*) - \tilde{D}(\hat{\bm \alpha}^*)) }, \\
  \label{eq:Ud_score}
  \bm \eta^\top \tilde{\bm \alpha}^* \le \bm \eta^\top \hat{\bm \alpha}^* + \|\bm \eta\|_2 \sqrt{2n\gamma^{-1}(\tilde{P}(\hat{\bm w}^*) - \tilde{D}(\hat{\bm \alpha}^*)) }.
\end{align}
From (\ref{eq:kkt_wj}), (\ref{eq:defi_diff_gap}),
(\ref{eq:Ld_score}), the assumptions and
since $\partial \psi^*_j$ is a monotonically increasing function because
$\psi^*_j $ is a convex function,
$\tilde{w}^*_j$ is bounded as by using Lemma \ref{lemma:score_bound_by_shpere}
\begin{align*}
  \tilde{w}^*_j
  &\ge \inf \partial \psi^*_j(n^{-1} \tilde{\bm z}_{ \cdot j}^\top \tilde{\bm \alpha}^* ) \\
  &\ge \inf \partial \psi^*_j
  \left( \max \left\{ \inf_{\bm \alpha \in {\rm dom} \tilde{D}(\bm \alpha) } n^{-1} \tilde{\bm z}_{ \cdot j}^\top \bm \alpha,~~
  \min_{\bm \alpha \in \Theta_{\tilde{\bm \alpha}^*} }
  n^{-1} \tilde{\bm z}_{ \cdot j}^\top \bm \alpha
  \right\}  \right) \\
  &\ge \inf \partial \psi^*_j \left( \max \left\{
  \inf_{\bm \alpha \in {\rm dom} \tilde{D}(\bm \alpha) } n^{-1} \tilde{\bm z}_{ \cdot j}^\top {\bm \alpha},~~
  n^{-1} \tilde{\bm z}_{ \cdot j}^\top \hat{\bm \alpha}^* -
  \|\tilde{\bm x}_{ \cdot j}\|_2 \sqrt{2 \gamma^{-1} \tilde{G}(\hat{\bm w}^*, \hat{\bm \alpha}^*) } \right\}  \right).
\end{align*}
Thus, $ \tilde{w}^*_j \ge L_D (w^*_j ) $.
Similarly, from (\ref{eq:kkt_wj}), (\ref{eq:defi_diff_gap}),
(\ref{eq:Ud_score}), the assumptions and
since $\partial \psi^*_j$ is a monotonically increasing function,
$\tilde{w}^*_j$ is bounded as:
\begin{align*}
  \tilde{w}^*_j
  &\le \sup \partial \psi^*_j(n^{-1} \tilde{\bm z}_{ \cdot j}^\top \tilde{\bm \alpha}^* ) \\
  &\le \sup \partial \psi^*_j
  \left( \min \left\{ \sup_{\bm \alpha \in {\rm dom} \tilde{D}(\bm \alpha) } n^{-1} \tilde{\bm z}_{ \cdot j}^\top \bm \alpha, ~~
  \max_{\bm \alpha \in \Theta_{\tilde{\bm \alpha}^*} }
  n^{-1} \tilde{\bm z}_{ \cdot j}^\top \bm \alpha
  \right\}  \right) \\
  &\le \sup \partial \psi^*_j \left( \min \left\{
  \sup_{\bm \alpha \in {\rm dom} \tilde{D}(\bm \alpha) } n^{-1} \tilde{\bm z}_{ \cdot j}^\top {\bm \alpha}, ~~
  n^{-1} \tilde{\bm z}_{ \cdot j}^\top \hat{\bm \alpha}^* +
  \|\tilde{\bm x}_{ \cdot j}\|_2 \sqrt{2 \gamma^{-1} \tilde{G}(\hat{\bm w}^*, \hat{\bm \alpha}^*) } \right\}  \right).
\end{align*}
Thus, $ \tilde{w}^*_j \le U_D (w^*_j ) $.
Then we bound $\tilde{\alpha}^*_i$. From Lemma \ref{lemma:dual_gap_bound} we have
\begin{align*}
  \hat{\alpha}^*_i - \sqrt{2n\gamma^{-1} \tilde{G}(\hat{\bm w}^*, \hat{\bm \alpha}^*) }
  \le \tilde{\alpha}^*_i \le
  \hat{\alpha}^*_i + \sqrt{2n\gamma^{-1} \tilde{G}(\hat{\bm w}^*, \hat{\bm \alpha}^*) },
\end{align*}
and thus
\begin{align*}
   \tilde{\alpha}^*_i \ge \max \left\{
   \inf_{\alpha \in {\rm dom} \tilde{D}(\alpha) } \alpha_i,~~
   \hat{\alpha}^*_i - \sqrt{2n\gamma^{-1} \tilde{G}(\hat{\bm w}^*, \hat{\bm \alpha}^*)} \right\}, \\
   \tilde{\alpha}^*_i \le \min \left\{
   \sup_{\alpha \in {\rm dom} \tilde{D}(\alpha) } \alpha_i,~~
   \hat{\alpha}^*_i + \sqrt{2n\gamma^{-1} \tilde{G}(\hat{\bm w}^*, \hat{\bm \alpha}^*)} \right\}.
\end{align*}
The all quantities required for evaluating
$L_D(\tilde{w}^*_j), U_D(\tilde{w}^*_j), L_D(\tilde{\alpha}^*_i)$ and $ U_D(\tilde{\alpha}^*_i) $
are stored in the memory without
$\| \tilde{\bm x}_{\cdot j} \|_2$
and
$\tilde{G}(\hat{\bm w}^*, \hat{\bm \alpha}^*)$.
Since $\| \tilde{\bm x}_{\cdot j} \|_2$ is stored in the memory and
$\| \tilde{\bm x}_{\cdot j} \|_2 = \sqrt{ \| \hat{\bm x}_{\cdot j} \|_2^2 - \sum_{i \in \cM(\cdot, j) } (\tilde{x}_{ij}^2 - \hat{x}_{ij}^2  ) }  $,
$\| \tilde{\bm x}_{\cdot j} \|_2$ is computed with
$\cO(| \cM(\cdot, j) | ) $ time complexity and
form the assumptions, the computational cost of
$\tilde{G}(\hat{\bm w}^*, \hat{\bm \alpha}^*)$ is $\cO(|\cM|) $ and so
$L_D(\tilde{w}^*_j), L_P(\tilde{w}^*_j), L_D(\tilde{\alpha}^*_i), L_p(\tilde{\alpha}^*_i) $
are evaluated with $\cO (|\cM|) $ time complexity.

$(ii)$ The case where the penalty function has {\bf Property B}:

From {\bf Property B}, $\tilde{P}$ is $\lambda$-strongly convex and so
$ \tilde{\bm w}^* \in \Theta_{\tilde{\bm w}^*}
:= \{ \bm w ~|~ \| \bm w - \hat{\bm w}^* \|_2 \le \sqrt{2 \lambda^{-1} ( \tilde{P}(\hat{\bm w}) - \tilde{D}(\hat{\bm \alpha}) )} \} $
by using Lemma \ref{lemma:primal_gap_bound},
thereby
\begin{align*}
  \hat{w}^*_j - \sqrt{2 \lambda^{-1} \tilde{G}(\hat{\bm w}^*, \hat{\bm \alpha}^*) }
  \le \tilde{w}^*_j \le
  \hat{w}^*_j + \sqrt{2 \lambda^{-1} \tilde{G}(\hat{\bm w}^*, \hat{\bm \alpha}^*) }.
\end{align*}
Thus,
\begin{align*}
   \tilde{w}^*_j \ge
   \hat{w}^*_j - \sqrt{2\lambda^{-1} \tilde{G}(\hat{\bm w}^*, \hat{\bm \alpha}^*)}, \\
   \tilde{w}^*_j \le
   \hat{\alpha}^*_i + \sqrt{2\lambda^{-1} \tilde{G}(\hat{\bm w}^*, \hat{\bm \alpha}^*)} .
\end{align*}

Similarly, since $\Theta_{\tilde{\bm w}^*}$ is a sphere,
the lower and upper bounds of the inner product of arbitrary vector
$\bm \eta \in \RR^d$ and
$\tilde{\bm w}^*$ are given in closed form as follows by using Lemma \ref{lemma:score_bound_by_shpere},
\begin{align}
  \label{eq:Lp_score}
  \bm \eta^\top \tilde{\bm w}^* \ge \bm \eta^\top \hat{\bm w}^* - \|\bm \eta\|_2 \sqrt{2\lambda^{-1}(\tilde{P}(\hat{\bm w}^*) - \tilde{D}(\hat{\bm \alpha}^*)) }, \\
  \label{eq:Up_score}
  \bm \eta^\top \tilde{\bm w}^* \le \bm \eta^\top \hat{\bm w}^* + \|\bm \eta\|_2 \sqrt{2\lambda^{-1}(\tilde{P}(\hat{\bm w}^*) - \tilde{D}(\hat{\bm \alpha}^*)) }.
\end{align}
From KKT condition $\tilde{\alpha}^*_i \in - \partial \phi (\tilde{\bm z}_{i \cdot} ^\top \tilde{\bm w}^* ) $,
(\ref{eq:defi_diff_gap}),
(\ref{eq:Lp_score}) and
since $\partial \phi$ is a monotonically increasing function,
$\tilde{\alpha}^*_i$ is bounded as:
\begin{align*}
  \tilde{\alpha}^*_i
  &\ge \inf \left\{ - \partial \phi \left(\tilde{\bm z}_{i \cdot} ^\top \tilde{\bm w}^* \right) \right\} \\
  &\ge \inf \left\{ - \partial \phi \left( \max_{\bm w \in \Theta_{\tilde{\bm w}^*} } \tilde{\bm z}_{i \cdot} ^\top \bm w \right) \right\} \\
  &= \inf \left\{ - \partial \phi \left( \tilde{\bm z}_{i \cdot} ^\top \hat{\bm w}^* + \|\tilde{\bm x}_{i \cdot} \|_2  \sqrt{2\lambda^{-1} \tilde{G}(\hat{\bm w}^*, \hat{\bm \alpha}^*)}  \right) \right\}
\end{align*}
From KKT condition $\tilde{\alpha}^*_i \in - \partial \phi (\tilde{\bm z}_{i \cdot} ^\top \tilde{\bm w}^* ) $,
(\ref{eq:defi_diff_gap}),
(\ref{eq:Up_score}) and
since $\partial \phi$ is a monotonically increasing function,
$\tilde{\alpha}^*_i$ is bounded as:
\begin{align*}
  \tilde{\alpha}^*_i
  &\le \sup \left\{ - \partial \phi \left(\tilde{\bm z}_{i \cdot} ^\top \tilde{\bm w}^* \right) \right\} \\
  &\le \sup \left\{ - \partial \phi \left( \min_{\bm w \in \Theta_{\tilde{\bm w}^*} } \tilde{\bm z}_{i \cdot} ^\top \bm w \right) \right\} \\
  &= \sup \left\{ - \partial \phi \left( \tilde{\bm z}_{i \cdot} ^\top \hat{\bm w}^* - \|\tilde{\bm x}_{i \cdot} \|_2  \sqrt{2\lambda^{-1} \tilde{G}(\hat{\bm w}^*, \hat{\bm \alpha}^*)}  \right) \right\} \\
\end{align*}
The all quantities required for evaluating
$L_P(\tilde{w}^*_j), U_P(\tilde{w}^*_j), L_P(\tilde{\alpha}^*_i)$ and $ U_P(\tilde{\alpha}^*_i) $
are stored in the memory except
$\| \tilde{\bm x}_{i \cdot} \|_2$
and
$\tilde{G}(\hat{\bm w}^*, \hat{\bm \alpha}^*)$.
Since $\| \hat{\bm x}_{i \cdot} \|_2$ is stored in the memory and
$\| \tilde{\bm x}_{i \cdot } \|_2 = \sqrt{ \| \tilde{\bm x}_{i \cdot} \|_2^2 - \sum_{j \in \cM(i, \cdot) } (\tilde{x}_{ij}^2 - \hat{x}_{ij}^2  ) }  $,
$\| \tilde{\bm x}_{\cdot j} \|_2$ is computed with
$\cO(| \cM(\cdot, j) | ) $ and
from the assumptions, the computational cost of
$\tilde{G}(\hat{\bm w}^*, \hat{\bm \alpha}^*)$ is $\cO(|\cM|) $ and so
$L_D(\tilde{w}^*_j), L_P(\tilde{w}^*_j), L_D(\tilde{\alpha}^*_i), L_p(\tilde{\alpha}^*_i) $
are evaluated with $\cO (|\cM|) $ time complexity.

$(iii)$ The case where the loss function has {\bf Property A} and the penalty function has {\bf Property B}:

By using {\bf Property A} and {\bf Property B},
$L_D(\tilde{w}^*_j)$, $ U_D(\tilde{w}^*_j)$, $ L_D(\tilde{\alpha}^*_i)$, $ U_D(\tilde{\alpha}^*_i)$,
$L_P(\tilde{w}^*_j)$, $ U_P(\tilde{w}^*_j)$, $ L_P(\tilde{\alpha}^*_i)$ and $ U_P(\tilde{\alpha}^*_i) $
can be evaluated with $\cO (|\cM|) $ time complexity.
Thus,
$L_{PD}(\tilde{w}^*_j), U_{PD}(\tilde{w}^*_j), L_{PD}(\tilde{\alpha}^*_i), U_{PD}(\tilde{\alpha}^*_i) $
can be evaluated with $\cO (|\cM|) $ time complexity.
\end{proof}

\subsection{Proof of Theorem 2}

\begin{proof}
Consider partially optimizing the primal problem only w.r.t. a subset of primal solutions $ \{ \hat{\bm w}_j \}_{ j \in \cJ \subseteq [d] } $.
Since
\begin{align*}
\check{\bm w}^*_{\cJ} =
\argmin_{ \{w_j \}_{j \in \cJ } }
\frac{1}{n} \sum_{i \in [n]} \phi \left( \sum_{j \in \cJ } \tilde{z}_{ij} w_j  + \sum_{j \in [d] \setminus \cJ } \tilde{z}_{ij} w_j \right)
+ \sum_{j \in \cJ} \psi_j(w_j),
\end{align*}
unless $ \check{w}^*_j = \hat{w}^*_j $ for all $j \in \cJ $, because of the strong convexity of $\tilde{P}$, $\tilde{P}(\hat{\bm w}^*) > \tilde{P}(\check{\bm w}^* ) $ holds and therefore $\tilde{G} (\hat{\bm w}^*, \hat{\bm \alpha}^* ) > \tilde{G} (\check{\bm w}^*, \hat{\bm \alpha}^* ) $.
Thus, both the radius of $\Theta_{\tilde{\bm w}^*} $ and $\Theta_{\tilde{\bm \alpha}^* } $ (the same definition as in Theorem 1) 
based on the primal-dual feasible solution pair $ (\check{\bm w}^*, \hat{\bm \alpha}^*) $
is smaller than the radius of $\Theta_{\tilde{\bm w}^*} $ and $\Theta_{\tilde{\bm \alpha}^* } $
based on the primal-dual feasible solution pair $ (\hat{\bm w}^*, \hat{\bm \alpha}^*) $.
Similarly, consider partially optimizing the dual problem
only w.r.t. a subset of dual solutions $ \{ \hat{\alpha}_i \}_{ i \in \cI \subseteq [n] } $.
Since
\begin{align*}
\check{\bm \alpha}^*_{\cI} =
\argmax_{ \{ \alpha_i \}_{i \in \cI } } - \frac{1}{n} \sum_{i \in \cI} \phi^* ( - \alpha_i)
 - \sum_{j \in [d]} \psi^*_j \left( \frac{1}{n} \sum_{i \in \cI} \tilde{z}_{ij} \alpha_i + \frac{1}{n} \sum_{i \in [n] \setminus \cI } \tilde{z}_{ij} \alpha_i \right),
\end{align*}
unless $ \check{\alpha}^*_i = \hat{\alpha}^*_i $ for all $i\in \cI $, because of the strong concavity of $\tilde{D}$, $\tilde{D}(\hat{\bm \alpha}^*) < \tilde{D}(\check{\bm \alpha}^* ) $ holds and therefore
$\tilde{G} (\hat{\bm w}^*, \hat{\bm \alpha}^* ) > \tilde{G} (\hat{\bm w}^*, \check{\bm \alpha}^* ) $.
Thus, both the radius of $\Theta_{\tilde{\bm w}^*} $ and $\Theta_{\tilde{\bm \alpha}^* } $ based on the primal-dual feasible solution pair $ (\hat{\bm w}^*, \check{\bm \alpha}^*) $
is smaller than the radius of $\Theta_{\tilde{\bm w}^*} $ and $\Theta_{\tilde{\bm \alpha}^* } $ based on the primal-dual feasible solution pair $ (\hat{\bm w}^*, \hat{\bm \alpha}^*) $.
\end{proof}

\section{Bounds in Theorem~\ref{theo:main} for smoothed-hinge SVMs}
\label{app:smoothe-hinge-SVM}
In the case of smoothed-hinge SVM, the bounds in
Theorem~\ref{theo:main}
are written as
 \begin{align*}
  L_D(\tilde{{w}}_j^*) =
  \min\left\{
  n^{-1} \sum{}_{i \mid \tilde{z}_{ij} < 0} \tilde{z}_{ij}, 
  n^{-1} \lambda^{-1} \tilde{\bm z}_{\cdot j}^\top \hat{\bm \alpha}^* - \lambda^{-1} \|\tilde{\bm x}_j\|_2 \sqrt{2 \gamma^{-1} \tilde{G}(\hat{\bm {w}}^*, \hat{\bm \alpha}^*)}
  \right\}
  \\
  U_D(\tilde{{w}}_j^*) = 
  \max\left\{
  n^{-1} \sum{}_{i \mid \tilde{z}_{ij} > 0} \tilde{z}_{ij}, 
  n^{-1} \lambda^{-1} \tilde{\bm z}_{\cdot j}^\top \hat{\bm \alpha}^* + \lambda^{-1} \|\tilde{\bm x}_j\|_2 \sqrt{2 \gamma^{-1} \tilde{G}(\hat{\bm {w}}^*, \hat{\bm \alpha}^*)}
  \right\}
   \end{align*}
   \begin{align*}
  L_D(\tilde{\alpha}_i^*) =
  \min\left\{
  0, \hat{\alpha}_i - \sqrt{2n \gamma^{-1} \tilde{G}(\hat{\bm {w}}^*, \hat{\bm \alpha}^*)}
  \right\},
  ~~~
  U_D(\tilde{\alpha}_i^*) = 
  \max\left\{
  1, \hat{\alpha}_i + \sqrt{2n \gamma^{-1} \tilde{G}(\hat{\bm {w}}^*, \hat{\bm \alpha}^*)}
  \right\},
  \end{align*}
    \begin{align*}
     \\
     L_P(\tilde{{w}}_j^*) = \hat{{w}}_j - \sqrt{2 \lambda^{-1} \tilde{G}(\hat{\bm {w}}^*, \hat{\bm \alpha}^*)}
  ~~~
  U_P(\tilde{{w}}_j^*) = \hat{{w}}_j + \sqrt{2 \lambda^{-1} \tilde{G}(\hat{\bm {w}}^*, \hat{\bm \alpha}^*)},
\end{align*}
\begin{align*}
  L_P(\tilde{\alpha}_i^*) =
  \begin{cases}
   0, & (\tilde{\bm z}_{i \cdot}^\top \hat{\bm {w}}^* + \|\tilde{\bm x}_{i \cdot}\|_2 \sqrt{2 \lambda^{-1} \tilde{G}(\hat{\bm {w}}^*, \hat{\bm \alpha}^*)} > 1), \\
   1, & (\tilde{\bm z}_{i \cdot}^\top \hat{\bm {w}}^* + \|\tilde{\bm x}_{i \cdot}\|_2 \sqrt{2 \lambda^{-1} \tilde{G}(\hat{\bm {w}}^*, \hat{\bm \alpha}^*)} < 1 - \gamma), \\
   - \gamma^{-1}(1 - \tilde{\bm z}_{i \cdot}^\top \hat{\bm {w}}^* - \|\tilde{\bm x}_{i \cdot}\|_2 \sqrt{2 \lambda^{-1} \tilde{G}(\hat{\bm {w}}^*, \hat{\bm \alpha}^*)}), & \text{(otherwise)},
  \end{cases}
\end{align*}
\begin{align*}
  U_P(\tilde{\alpha}_i^*) = 
  \begin{cases}
   0, & (\tilde{\bm z}_{i \cdot}^\top \hat{\bm {w}}^* - \|\tilde{\bm x}_{i \cdot}\|_2 \sqrt{2 \lambda^{-1} \tilde{G}(\hat{\bm {w}}^*, \hat{\bm \alpha}^*)} > 1), \\
   1, & (\tilde{\bm z}_{i \cdot}^\top \hat{\bm {w}}^* - \|\tilde{\bm x}_{i \cdot}\|_2 \sqrt{2 \lambda^{-1} \tilde{G}(\hat{\bm {w}}^*, \hat{\bm \alpha}^*)} < 1 - \gamma), \\
   - \gamma^{-1}(1 - \tilde{\bm z}_{i \cdot}^\top \hat{\bm {w}}^* + \|\tilde{\bm z}_{i \cdot}\|_2 \sqrt{2 \lambda^{-1} \tilde{G}(\hat{\bm {w}}^*, \hat{\bm \alpha}^*)}), & \text{(otherwise)}.
  \end{cases}
    \end{align*}

%\input{appC}
%--------------------------------------- 

\end{document}